%% file: arxiv.tex
\documentclass{arxiv}
\PassOptionsToPackage{table}{xcolor}

\input{math_commands.tex}

\usepackage{algorithm}
\usepackage{algorithmic}
\usepackage{amsmath}
\usepackage{amssymb}
\usepackage{amsthm}
\usepackage{bbding} % Required for \XSolidBrush
\usepackage{booktabs} % for professional tables
\usepackage{color}
\usepackage{colortbl}
\usepackage{enumitem}

\usepackage{xcolor}
\definecolor{pink}{HTML}{747199}

\usepackage{graphicx}
\usepackage{hyperref}
\hypersetup{colorlinks=true, citecolor=pink, urlcolor=pink, linkcolor=pink}

\usepackage{multicol,multirow}
\usepackage{natbib}
\usepackage{pifont}
\usepackage{setspace}
\usepackage{subfigure}
\usepackage{threeparttable}
\usepackage{url}
\usepackage{wrapfig}

\usepackage{caption}
\usepackage{longtable}
\usepackage{arydshln}
\definecolor{forestgreen}{RGB}{79,173,91}
\definecolor{forestyellow}{RGB}{245,195,66}
\definecolor{myblue}{rgb}{0.82, 0.94, 0.75}
\definecolor{mygold}{rgb}{1, 0.92, 0.56}
\definecolor{mylightblue}{rgb}{0.70, 0.83, 0.96}
\definecolor{mylightyellow}{rgb}{0.96, 0.88, 0.49}
\definecolor{mylightpink}{rgb}{0.93, 0.79, 0.80}
\definecolor{mygray}{HTML}{f0f0f0}

\definecolor{myblue2}{RGB}{68,114,196}

\definecolor{myred2}{RGB}{192,0,0}

\definecolor{customorange}{RGB}{237,125,49}

\theoremstyle{plain}
\newtheorem{theorem}{Theorem}[section]
\newtheorem{proposition}[theorem]{Proposition}

\theoremstyle{definition}

\newtheorem{assumption}[theorem]{Assumption}
\newtheorem{remark}[theorem]{Remark}

\newcommand{\ie}{\textit{i.e.},}
\newcommand{\eg}{\textit{e.g.},}

\newcolumntype{V}[1]{>{\centering\arraybackslash}m{#1}}

\usepackage[most]{tcolorbox}

\tcbuselibrary{breakable} % 引入 breakable 库
\tcbuselibrary{skins}
\newtcolorbox[
    use counter=tcboxcounter,number within=section
]{mybox}[3][]{
    left=3pt,
    right=4pt,
    breakable,
    enhanced,
    title=#2 \thetcbcounter: #3,
    #1
}

\definecolor{startBlue}{HTML}{1628a7} % 起点：皇家科技蓝
\definecolor{endPurple}{HTML}{8b16aa} % 终点：紫罗兰色

\ExplSyntaxOn
\cs_new:Npn \__gradien_text:nnn #1#2#3 
{
    \tl_set:Nn \l_tmpa_tl { #3 }
    \int_set:Nn \l_tmpa_int { \tl_count:N \l_tmpa_tl }
    \int_step_inline:nn { \l_tmpa_int } {
        \textcolor{#2!\fp_eval:n{##1 * 100/\l_tmpa_int}!#1}{\tl_item:Nn \l_tmpa_tl {##1}}
        \allowbreak
    }
}
\NewDocumentCommand{\gradienttext}{mm+m}{\__gradien_text:nnn{#1}{#2}{#3}}
\ExplSyntaxOff

\title{\gradienttext{startBlue}{endPurple}{Rubric~Judgments~Lack~Uncertainty}: Conformal Weighting for LLM Evaluation}

\author[1,2]{Licheng Pan}
\author[3]{Haocheng Yang}
\author[4]{Haoxuan Li}
\author[1]{Yuzhou Chen}
\author[1]{Yichen Sun}
\author[2]{Yuan Lu}
\author[2,5]{Xiaoxi Li}
\author[6*]{Hao Wang}
\author[1*]{Zhixuan Chu}
\author[1]{Kui Ren}

\affiliation[1]{Zhejiang University}
\affiliation[2]{Xiaohongshu Inc}
\affiliation[3]{National University of Singapore}
\affiliation[4]{Peking University}
\affiliation[5]{Renmin University of China}
\affiliation[6]{Mohamed bin Zayed University of Artificial Intelligence}

\contribution[*]{Corresponding author.}

\metadata[Contact]{licheng.pan@zju.edu.cn, haohaow@zju.edu.cn, zhixuanchu@zju.edu.cn}
\metadata[Code]{\href{https://anonymous.4open.science/r/URW-27E4}{https://anonymous.4open.science/r/URW-27E4}.}

\abstract{
Rubric-based evaluation is widely used to assess large language model (LLM) outputs by decomposing quality into multiple criteria and assigning each a binary judgment from an LLM verifier.
However, judgment certainty varies substantially across criteria.
Existing methods uniformly aggregate these judgments into a final score, implicitly treating every criterion as equally certain and overlooking this heterogeneity.
As a result, uncertain judgments remain incorporated into reward estimation without distinction, potentially yielding misleading training signals for downstream reinforcement learning (RL).
To address this limitation, we propose \textbf{Uncertainty-aware Rubric Weighting (URW)}, a training-free framework that quantifies criterion-level judgment uncertainty via conformal calibration.
Specifically, URW elicits a soft uncertainty signal from a fixed pool of paraphrased judgment templates and converts it into a per-criterion weight using split-conformal $p$-values with provable coverage guarantees.
Extensive experiments demonstrate that URW consistently improves pairwise preference accuracy and delivers stronger downstream RL alignment.
}

\begin{document}
\maketitle

\section{Introduction}

Evaluating open-ended outputs of large language models (LLMs) increasingly relies on rubric-based evaluation, in which response quality is decomposed into explicit criteria and assessed by LLM verifiers~\citep{arora2025healthbench,li2025generation}.
Generally, the verifier renders binary pass or fail judgments on these prompt-specific criteria, which are then aggregated into a response-level score~\citep{cook2024tick}.
The resulting scores underpin downstream applications such as ranking candidate responses~\citep{lambert2025rewardbench}, providing reward signals for reinforcement learning (RL)~\citep{mu2024rulebased,gunjal2026rubrics,shan2026rubricrlsurvey} and verifying agent systems~\citep{wang2025video,chen2026causalgame,liu2025llmpowered,xu2026draftingagentsjusttokens}.
Consequently, reliable aggregation of rubric judgments is essential for both LLM response evaluation and policy optimization.

Despite rapid progress in rubric generation and evolution~\citep{cook2024tick,chen2026genrubric,yang2026rubricarmoradversarialevolutionimproves}, aggregation of rubric judgments heavily relies on uniform weighting. This uniformity reflects a critical limitation: rubric judgments carry no signal of uncertainty.
Verifiers output a binary decision per criterion, without any confidence signal that aggregation could use for weighting. However, the difficulty of rendering a reliable judgment varies substantially across criteria~\citep{song2026illusion,chen2026gamut}.
For example, format-level checks are typically clear-cut, while criteria involving tone, helpfulness, or content quality are often inherently subjective and harder to judge~\citep{pombal2026selfpref}.
Under uniform aggregation, both types of criteria are collapsed into binary votes and given equal weight. When two responses are otherwise closely matched, an uncertain judgment on one subjective criterion can be enough to invert the overall preference.

However, existing work on uncertainty quantification for LLM judges does not fully address the need for per-criterion, uncertainty-aware weighting.
One line of work uses conformal prediction to estimate response-level uncertainty, either as a prediction interval over the aggregated rating~\citep{sheng2025interval,liu2026conformalconsensus} or as a confidence flag~\citep{wagner2024blackbox}, characterizing uncertainty in the final rating rather than at the criterion level.
Another line of work derives confidence from decoding-time consistency, but requires non-zero-temperature sampling~\citep{li2026temperature,cao2026clustered} or heterogeneous verifier committees~\citep{fathullah2025generalised}, both incompatible with common deterministic, single-verifier settings.
A third line of work trains dedicated uncertainty modules on top of LLM judges, either by fine-tuning uncertainty-aware evaluators~\citep{xie2025confilm} or by learning calibration probes~\citep{radharapu2026calibrating}, requiring annotated training data and producing additional output labels rather than a coverage-controlled scalar weight compatible with weighted-sum aggregation.
Overall, no existing method provides a training-free, coverage-controlled, per-criterion uncertainty weight that can be applied to any binary-rubric verifier at inference time.

To address this limitation, we propose \textbf{Uncertainty-aware Rubric Weighting (URW)}, a training-free framework that quantifies per-criterion uncertainty of any verifier and turns it into a distribution-free aggregation weight.
In the offline calibration phase, URW elicits a soft uncertainty signal from the verifier via a fixed pool of paraphrased judgment templates, and measures how well this signal aligns with gold labels on a small calibration set.
In the online weighting phase, URW converts the soft signal for each new criterion into a split-conformal $p$-value, which directly serves as the per-criterion aggregation weight and inherits a distribution-free guarantee: the probability of assigning a spuriously low weight to a correct judgment is provably controlled.
Extensive experiments across diverse preference benchmarks and verifier families demonstrate that URW consistently improves both pairwise preference accuracy and downstream RL alignment.

\paragraph{Contributions.} The contributions of this work are summarized as follows.
\ding{182} \textbf{We identify an overlooked limitation of rubric-based LLM evaluation:} rubric judgments carry no per-criterion uncertainty, so uniform aggregation can let a single unreliable judgment on a closely contested criterion invert the overall preference between two otherwise comparable responses.
\ding{183} \textbf{We propose URW, a training-free framework for uncertainty-aware rubric weighting.} URW couples paraphrase-template ensembling with split-conformal calibration, producing a per-criterion $p$-value that serves as an aggregation weight and provably controls the chance of spuriously downweighting a correct judgment.
\ding{184} \textbf{We conduct comprehensive experiments on five preference benchmarks to evaluate the efficacy of URW,} showing that it consistently improves pairwise preference accuracy over diverse baselines and delivers stronger downstream RL alignment.

\section{Preliminaries}
\label{sec:preliminaries}

\subsection{Rubric-Based Preference Evaluation}\label{sec:prelim_rubric}

Rubric-based preference evaluation has recently emerged as an interpretable alternative or complement to conventional reward modeling~\citep{wang2026causalrm,wang2026icmlimplicitrm,pan2026selectiverm,pan2026uncertainty}. 
Specifically, given a user prompt $q\in\mathcal{Q}$ and a candidate response $y\in\mathcal{Y}$, a rubric evaluation~\citep{li2025generation,gu2026survey} decomposes the judgment into $\mathrm{K}$ prompt-specific binary criteria $r(q)=\{c_1,\dots,c_\mathrm{K}\}$, each phrased so that the correct answer is either $1$ (satisfied) or $0$ (violated). An LLM-based verifier $v:\mathcal{Q}\times\mathcal{Y}\times\mathcal{C}\to\{0,1\}$ receives each triple independently and produces a binary judgment $v(q,y,c_k)$ for each criterion.
Given the per-criterion verifier outputs and a scalar criterion-level weight vector $\boldsymbol{w}=(w_1,\dots,w_{\mathrm{K}})$ with $w_k\ge 0$ and $\sum_k w_k=1$, the response-level quality score of $y$ is obtained by weighted linear aggregation of the per-criterion judgments as follows:
\begin{equation}\label{eq:weighted_sum}
\mathcal{J}\left(y\mid\boldsymbol{w},q,r\right)=\sum_{k=1}^{\mathrm{K}} w_k \cdot v(q,y,c_k).
\end{equation}
Existing rubric-based LLM judges~\citep{gunjal2026rubrics,viswanathan2026checklists,arora2025healthbench} typically instantiate $\boldsymbol{w}$ uniformly, \ie $w_k=1/\mathrm{K}$, treating every criterion as equally reliable. While simple and effective, it ignores that a verifier's per-criterion judgments vary substantially in confidence: criteria involving subjective or context-dependent quality attributes (\eg tone or helpfulness) are inherently harder to judge reliably, whereas criteria with objective, easily verifiable conditions (\eg format or length constraints) admit near-deterministic decisions.
As a result, uniform aggregation treats every judgment as equally trustworthy regardless of underlying difficulty, so a single uncertain judgment carries the same influence on the aggregated score as a highly confident one.

\subsection{Split Conformal Prediction}\label{sec:prelim_scp}

Conformal prediction~\citep{shafer2008tutorial} is a general, distribution-free framework for uncertainty quantification that produces prediction sets with finite-sample coverage guarantees under mere exchangeability. Its most widely used instantiation, split conformal prediction (SCP)~\citep{scp}, achieves computational tractability by reserving a held-out calibration split, avoiding the need to refit the model for each test sample.
Since SCP treats the underlying predictor as a black box and requires no retraining, it is naturally suited to calibrating an LLM verifier's per-criterion certainty.
Formally, let $\mathcal{D}_\mathrm{cal}=\{(q_i,y_i,c_i,v_i^*)\}_{i=1}^{\mathrm{N}_\mathrm{cal}}$ be a held-out calibration set of prompt-response-criterion triples with gold labels $v_i^*\in\{0,1\}$. We assume that $\mathcal{D}_\mathrm{cal}$ is exchangeable with a new test triple $(q,y,c_k)$ whose gold label $v^*$ is unknown. Given the nonconformity score $s(q,y,c_k)=|v(q,y,c_k)-v^*|$, which measures the discrepancy between the verifier's output and the gold label, and a miscoverage rate $\alpha\in(0,1)$, SCP forms the prediction set as follows:
\begin{equation}\label{eq:scp_set}
\mathcal{I}_\alpha(q,y,c_k)=\big\{v \big\vert |v(q,y,c_k)-v|\le \tau_\alpha\big\},\qquad
\tau_\alpha=\text{Quantile}\left(\{s_i\}_{i=1}^{\mathrm{N}_\mathrm{cal}}, \tfrac{\lceil(1-\alpha)(\mathrm{N}_\mathrm{cal}+1)\rceil}{\mathrm{N}_\mathrm{cal}}\right),
\end{equation}
where $s_i=|v(q_i,y_i,c_i)-v_i^*|$. Under mere exchangeability of $\mathcal{D}_\mathrm{cal}\cup\{(q,y,c_k,v^*)\}$, SCP admits the finite-sample marginal coverage guarantee $\mathbb{P}\big[v^*\in \mathcal{I}_\alpha(q,y,c_k)\big]\ge 1-\alpha$.

\subsection{Problem Definition}\label{sec:prelim_problem}

% This work investigates the \textbf{uncertainty-aware rubric weighting} problem, which aims to model an instance-adaptive weight $\hat{w}_k(q,y)$ that tracks the verifier's per-criterion certainty, replacing the uniform weight $w_k=1/\mathrm{K}$ commonly adopted in existing aggregation methods so that uncertain judgments are automatically downweighted at inference.
% Specifically, let $\pi_k(q,y)\triangleq\mathbb{P}[v(q,y,c_k)=v^*(q,y,c_k)]$ denotes the verifier's certainty on criterion $c_k$, where $v^*$ denotes the ground truth judgment. Since $\pi_k$ is heterogeneous across criteria and cannot be read off any single deterministic decoding pass, no fixed $\boldsymbol{w}$ can distinguish reliable from unreliable judgments.
% \textbf{Thus, the core problem is to construct $\hat{w}_k(q,y)$ from the observed verifier output that effectively tracks $\pi_k(q,y)$, assigning low weight to unreliable judgments and high weight to reliable ones.}

This work investigates the \textbf{uncertainty-aware rubric weighting} problem, specifically the challenge of constructing instance-adaptive weights $\hat{w}_k(q,y)$ for different rubric criteria from the outputs of a deterministic LLM verifier. These weights can be used in place of uniform weights in rubric aggregation. Ideally, they reflect the relative certainty of individual rubric judgments, so that judgments with greater uncertainty exert less influence on the response-level aggregation.
However, a deterministic verifier provides only a binary judgment for each criterion and does not directly expose its certainty. \textbf{Thus, the core problem is to construct $\hat{w}_k(q,y)$ from observable verifier behavior that captures uncertainty in individual rubric judgments.}

\section{Methodology}

\subsection{Motivation}\label{sec:motivation}

% Per-criterion certainty is unobservable under a single decoding pass. Each binary output $v(q,y,c_k)$ from an LLM verifier carries no probability, logit, or confidence signal\footnote{We refer to the LLM verifier as an API-based black-box verifier.}, so the aggregation in Eq.~(\ref{eq:weighted_sum}) cannot distinguish a certain judgment from an uncertain one.
% A natural remedy of sampling the verifier at a higher temperature is also unavailable, as it departs from the deterministic decoding commonly adopted by reference protocols and downstream RL.

% Consequently, uniform aggregation is fragile to individual uncertain judgments. Under $w_k=1/\mathrm{K}$, an uncertain judgment enters the aggregated score with the same influence as a confident one. Consider two nearly-tied responses that agree on most criteria but are judged differently on a single subjective one (\eg tone or helpfulness): if the verifier is in fact uncertain there, that lone judgment already tips the aggregated preference, with no mechanism to discount its contribution.

Judgment certainty for rubric criteria is unobservable from the LLM verifier. Each binary output $v(q,y,c_k)$ carries no probability, logit, or confidence signal\footnote{We refer to the LLM verifier as an API-based black-box verifier.}. To obtain such confidence signals, a natural solution is to sample the verifier at higher temperature. However, this departs from the deterministic decoding commonly adopted by reference verification protocols and is incompatible with downstream policy optimization algorithms that require stable reward signals.

This missing uncertainty makes uniform aggregation fragile. Under $w_k=1/\mathrm{K}$, every judgment enters the aggregated score in Eq.~(\ref{eq:weighted_sum}) with equal influence regardless of its certainty. Consider two nearly-tied responses that agree on most criteria but are judged differently on a single subjective one (\eg tone or helpfulness): if the verifier is in fact uncertain there, that lone judgment already tips the aggregated preference, with no mechanism to discount its contribution.

\textbf{Case study.} To substantiate the above analysis, we present an illustrative example in Figure~\ref{fig:case_study}. An anxious user asks how to handle an interview: response $y_1$ acknowledges the worry with warm, actionable advice, while response $y_2$ gives correct steps in a curt, dismissive tone. The two responses tie on the objective criteria ($c_1,c_2$) and split the rest, so the preference hinges entirely on the empathetic-tone criterion $c_5$. Here the verifier credits the curt response $y_2$, inverting the aggregate to $\mathcal{J}(y_1\mid q,r)=\tfrac{3}{5} < \mathcal{J}(y_2\mid q,r)=\tfrac{4}{5}$ and contradicting the human preference $y_1\succ y_2$.
This showcases that uniform aggregation is defenseless against a verifier's uncertain judgments.

\begin{figure}[t]
\includegraphics[width=\linewidth]{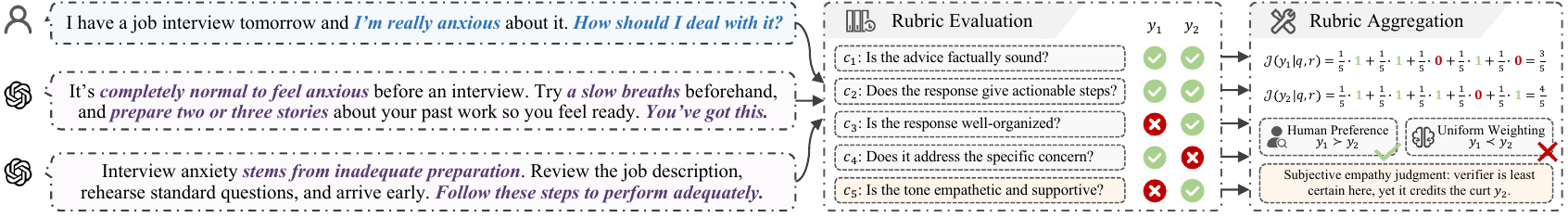}
\caption{A case study where a single uncertain judgment on a subjective criterion inverts the aggregated preference under uniform weighting.}
\label{fig:case_study}
\end{figure}

Some might note prior work on uncertainty quantification for LLM judges; however, its utility for rubric weighting remains under-explored. Existing methods either estimate response-wise uncertainty rather than per criterion~\citep{sheng2025interval,liu2026conformalconsensus,wagner2024blackbox}, rely on higher-temperature sampling or verifier committees that depart from the deterministic reference decoding~\citep{li2026temperature,cao2026clustered,fathullah2025generalised}, or train dedicated modules that require annotated data and emit auxiliary labels instead of an criterion-level weight~\citep{xie2025confilm,radharapu2026calibrating}. \textbf{Therefore, deriving a training-free, per-criterion uncertainty weight that preserves the deterministic verifier decoding remains an open and critical challenge.}

\subsection{Quantifying Verifier Uncertainty via Template Ensembling}\label{sec:template_ensemble}

In this section, we quantify a per-criterion uncertainty signal from the verifier without perturbing its decoding. The key idea is that a reliable judgment should be stable under semantics-preserving rephrasings of the criterion prompt, whereas an uncertain one flips easily.
Therefore, we probe this stability with a pool of paraphrased judgment templates rather than by raising the decoding temperature.
Specifically, let $\mathcal{T}=\{t_1,\dots,t_\mathrm{N}\}$ be a set of $\mathrm{N}$ judgment prompts (see Appendix~\ref{app:templates} for the exact pool), where each paraphrase has been checked multiple times by human experts\footnote{Human experts refer to our author team, who jointly reviewed and cross-checked the template pool.} to confirm it preserves the original criterion's meaning. This procedure ensures cross-template disagreement reflects verifier uncertainty rather than a shift in what is being asked.
Then we apply the deterministic decoding under each template, obtaining per-template binary judgments $v_{k,n}=v(q,y,c_k\mid t_n)$ for criterion $c_k$.
Averaging these judgments yields the soft judgment score:
\begin{equation}\label{eq:vhat}
\hat v_k := \hat v(q,y,c_k)=\frac{1}{\mathrm{N}}\sum_{n=1}^\mathrm{N} v_{k,n}\in\left\{0,\tfrac{1}{\mathrm{N}},\dots,1\right\},
\end{equation}
which we use as the per-criterion uncertainty signal.
Here $\hat v_k$ measures how uncertain the verifier is on the criterion: it approaches $0$ or $1$ when the verifier is confident and approaches $0.5$ when it is maximally uncertain.
To calibrate this raw signal against ground-truth judgments later, we reuse the held-out calibration set $\mathcal{D}_\mathrm{cal}$ from Section~\ref{sec:prelim_scp} to build a reference distribution of nonconformity scores.
Specifically, for each calibration sample we compute its soft score $\hat v_i=\hat v(q_i,y_i,c_i)$ via Eq.~(\ref{eq:vhat}) and its nonconformity $s_i = \big|\hat v_i - v_i^*\big|\in[0,1]$.
A small nonconformity score $s_i$ marks a verifier that is both confident and correct on sample $i$, and a large one otherwise flags an uncertain or incorrect judgment.
These scores form the calibration distribution $\mathcal{S}=\{s_i\}_{i=1}^{\mathrm{N}_\mathrm{cal}}$, so that a new judgment is scored by where its own nonconformity ranks within $\mathcal{S}$ at inference time.

\subsection{Calibrating Verifier Uncertainty via $p$-value Weighting}\label{sec:pvalue_weighting}

In this section, we convert each raw signal $\hat v_k$ into a per-criterion weight through split conformal calibration, which requires no additional training and yields a weight with a distribution-free validity guarantee.
The key idea is to reuse the SCP framework in Section~\ref{sec:prelim_scp}. Rather than thresholding the calibration scores $\mathcal{S}$ at a fixed miscoverage level to form a prediction set, we measure how a test judgment ranks within $\mathcal{S}$. This rank yields a conformal $p$-value~\citep{shafer2008tutorial}, which grades the judgment on a continuous scale and discounts criteria on which the verifier is uncertain.
Specifically, for a new test triple $(q,y,c_k)$ with soft score $\hat v_k$ and a candidate label $u\in\{0,1\}$, its test-side nonconformity is $s(\hat v_k, u)=|\hat v_k-u|$.
Then the $p$-value of $u$ given $\hat v_k$ is defined as
\begin{equation}\label{eq:pval_hyp}
p(u \mid \hat v_k)=\frac{1+\left\vert\{i: s_i\ge s(\hat v_k, u)\}\right\vert}{\mathrm{N}_\mathrm{cal}+1},
\end{equation}
which indicates the calibration-grounded plausibility that $u$ is the true label.
Taking $\tilde v_k=\mathrm{round}(\hat v_k)$ as the verifier's voted label, we define the raw per-criterion weight as the $p$-value of that vote,
\begin{equation}\label{eq:raw_weight}
\tilde w_k = p(\tilde v_k\mid\hat v_k)
=
\begin{cases}
p(1\mid\hat v_k), & \hat v_k\ge 0.5,\\
p(0\mid\hat v_k), & \hat v_k<0.5.
\end{cases}
\end{equation}
This weight has two desirable properties. First, it is monotone in the distance $|2\hat v_k-1|$ from the decision boundary $\hat v_k=0.5$ (at which the vote $\tilde v_k=\mathrm{round}(\hat v_k)$ flips): it is smallest when the vote lies at the decision boundary and increases as the verifier's template agreement moves toward either endpoint, so uncertain judgments are downweighted automatically. Second, as a conformal $p$-value it inherits a finite-sample validity guarantee on correct judgments as stated in Theorem~\ref{thm:validity}.
\begin{theorem}[Validity of Uncertainty Weight]\label{thm:validity}
Suppose the calibration triples in $\mathcal{D}_\mathrm{cal}$ and the test triple $(q,y,c_k,v^*)$ are exchangeable under the same verifier and template pool $\mathcal{T}$. Then, for every $\alpha\in(0,1)$, we have $\mathbb{P}\big[\tilde w_k\le\alpha,\ \tilde v_k=v^*\big]\le\alpha$.
\end{theorem}

Theorem~\ref{thm:validity} guarantees that the $p$-value-induced weight does not frequently assign a spuriously low value to a correct voted judgment, and the bound holds uniformly over $\alpha$, so no miscoverage rate need be fixed in advance.
We then finalize uncertainty-aware rubric weighting (URW) by normalizing the raw weights $\tilde w_k$ over the $\mathrm{K}$ criteria: $w_k=\tfrac{\tilde w_k}{\sum_{k'=1}^{\mathrm{K}}\tilde w_{k'}}$. And finally per-criterion judgments are re-weighted and aggregated with the weights $w_k$. The resulting response-level score is given by
\begin{equation}\label{eq:urw_score}
\mathcal{J}(y\mid \boldsymbol{w},q,r)=\sum_{k=1}^{\mathrm{K}} w_k \tilde v_k = \sum_{k=1}^{\mathrm{K}} w_k \cdot\mathrm{round}\left(\frac{1}{\mathrm{N}}\sum_{n=1}^\mathrm{N} v(q,y,c_k\mid t_n)\right).
\end{equation}

\section{Experiments}\label{sec:experiments}

In this section, we validate the efficacy of URW through a comprehensive empirical evaluation, centered on five research questions (RQs) as follows:
\begin{enumerate}[leftmargin=*]
    \item \textbf{RQ1}: \textit{Does URW improve rubric-aggregation accuracy?} In Section~\ref{subsec:main_results}, we compare URW against uncertainty-aware weighting baselines on five preference benchmarks.
    \item \textbf{RQ2}: \textit{Why does it work?} In Section~\ref{subsec:ablation}, we perform an ablation study on the contribution of template ensembling and distribution-free calibration.
    \item \textbf{RQ3}: \textit{Does it generalize across verifiers and rubric sources?} In Section~\ref{subsec:generalization}, we evaluate its compatibility with verifiers of varying parameter scales and with different rubric sources.
    \item \textbf{RQ4}: \textit{Is it sensitive to the template count?} In Section~\ref{subsec:hyperparameter}, we analyze its sensitivity to the template count $\mathrm{N}$ across different verifiers.
    \item \textbf{RQ5}: \textit{Does it improve downstream RLVR?} In Section~\ref{subsec:downstream_rlhf}, we use URW to aggregate the rubric rewards that drive GRPO training and evaluate the resulting policy.
\end{enumerate}

\subsection{Experimental Setup}\label{subsec:setup}

\textbf{Datasets.}
We evaluate URW on five pairwise preference benchmarks across different domains, including RubricBench~\citep{zhou2026rubricbench}, ResearchQA~\citep{yifei2026researchqa}, RM-Bench~\citep{liu2025rmbench}, RewardBench~\citep{lambert2025rewardbench}, and RewardBench 2~\citep{malik2026rewardbench2}. For RewardBench and RewardBench 2, we adopt their most challenging subsets, namely Chat-Hard and Precise-IF, respectively. Each benchmark supplies a response pair $(y_\mathrm{A},y_\mathrm{B})$ per prompt together with a gold human preference label. For RM-Bench, whose instances list multiple chosen and rejected responses, we form pairs by diagonal matching. RubricBench and ResearchQA come with human-sourced rubrics, whereas for the remaining benchmarks we generate rubrics with RubricARM-8B-Rubric~\citep{xu2026alternating}.
For calibration in URW, we draw from the domain-diverse RUBRIC-ARROW-Judge-SFT rubric-annotation corpus~\citep{jiang2026rubricarrow}, whose broad domain coverage supports the exchangeability our method assumes. We stratify by domain and sample $200$ prompts, yielding about $2$k rubric-level annotations. The calibration set never intersects any evaluation split. Detailed dataset statistics and configurations are provided in Appendix~\ref{app:dataset}.

\textbf{Baselines.}
We benchmark URW against a naive reference and a suite of uncertainty-aware weighting baselines, including \textbf{Uniform} ($w_k=1/\mathrm{K}$), the signal-free floor that weights all criteria equally; \textbf{Beta-Binomial}, a Bayesian method that treats the multi-template judgments as binomial observations under a Beta prior and derives the weight from the resulting posterior; \textbf{Sampling Ensemble}~\citep{wang2023selfconsistency}, which draws multiple high-temperature decodings per criterion and weights by their vote ratio; and \textbf{SCP}, a distribution-free method that calibrates the same multi-template soft score $\hat v$ and hard-abstains via a global conformal quantile, together with its variant \textbf{Mondrian-SCP} that instead calibrates a separate quantile within each bin of $\hat v$.

\textbf{Implementation Details.}
Our main experiments employ three frozen verifiers across different model families, including \textbf{DeepSeek V4 Pro}~\citep{xu2026deepseek}, \textbf{Qwen 3.8 Max}~\citep{qwen38max}, and \textbf{GLM 5.3}~\citep{glm5}.
All reference verifier calls use temperature zero, JSON mode enabled, high reasoning effort, and a maximum of $32{,}768$ output tokens. The template pool contains $\mathrm{N}=6$ semantics-preserving paraphrases. For the Sampling Ensemble baseline, we re-decode each rubric multiple times at temperature $0.7$. For the Beta-Binomial baseline, we place a uniform $\mathrm{Beta}(1,1)$ prior on each criterion's satisfaction probability and take the posterior mean as the weight.
Our full method uses calibrated nonconformity score set $\mathcal{S}$ directly and requires no additional hyperparameter, whereas the SCP and Mondrian-SCP variants additionally use $\alpha=0.1$ and $6$ bins. As for metrics, we report pairwise accuracy (Acc) across all benchmarks, where ties receive only half credit. More implementation details are provided in Appendix~\ref{app:implementation}.

\begin{table}[t]
\centering
\begin{threeparttable}
\caption{Pairwise preference accuracy of URW versus baselines.}
\label{tab:main_result}
\setlength\tabcolsep{8.8pt}
\renewcommand\arraystretch{1.1}
\scriptsize
\begin{tabular}{lcccccccccc}
\toprule
 & \multicolumn{2}{c}{\textbf{RubricBench}}
 & \multicolumn{2}{c}{\textbf{ResearchQA}}
%  & \multicolumn{2}{c}{\textbf{JudgeBench}}
 & \multicolumn{2}{c}{\textbf{RM-Bench}}
 & \multicolumn{2}{c}{\textbf{RewardBench}}
 & \multicolumn{2}{c}{\textbf{RewardBench 2}} \\
\cmidrule(lr){2-3} \cmidrule(lr){4-5} \cmidrule(lr){6-7} \cmidrule(lr){8-9} \cmidrule(lr){10-11}
\textbf{Method}
 & Acc & RI & Acc & RI & Acc & RI & Acc & RI & Acc & RI \\
\midrule

\rowcolor[HTML]{f0f0f0}
\multicolumn{11}{l}{\textbf{\textit{Verifier: DeepSeek V4 Pro~\citep{xu2026deepseek}}}} \\
\rowcolor{white}
Uniform & 0.780 & -- & 0.633 & -- & 0.729 & -- & 0.766 & -- & 0.625 & -- \\
Beta-Binomial & 0.781 & 0.11\% $\uparrow$ & 0.639 & 0.93\% $\uparrow$ & 0.739 & 1.42\% $\uparrow$ & 0.772 & 0.72\% $\uparrow$ & 0.688 & 10.00\% $\uparrow$ \\
Sampling Ensemble & 0.781 & 0.11\% $\uparrow$ & 0.669 & 5.61\% $\uparrow$ & 0.711 & 2.48\% $\downarrow$ & 0.785 & 2.43\% $\uparrow$ & 0.731 & 17.00\% $\uparrow$ \\
SCP & 0.780 & 0.00\% & 0.645 & 1.87\% $\uparrow$ & 0.745 & 2.30\% $\uparrow$ & 0.784 & 2.29\% $\uparrow$ & 0.693 & 10.83\% $\uparrow$ \\
Mondrian-SCP & 0.793 & 1.62\% $\uparrow$ & 0.657 & 3.74\% $\uparrow$ & 0.762 & 4.61\% $\uparrow$ & 0.786 & 2.58\% $\uparrow$ & 0.719 & 15.00\% $\uparrow$ \\
\rowcolor[HTML]{ecf0ff}
\textbf{URW} & \textbf{0.802} & \textbf{2.79\% $\uparrow$} & \textbf{0.698} & \textbf{10.28\% $\uparrow$} & \textbf{0.792} & \textbf{8.69\% $\uparrow$} & \textbf{0.810} & \textbf{5.72\% $\uparrow$} & \textbf{0.759} & \textbf{21.50\% $\uparrow$} \\
\hdashline
\rowcolor[HTML]{f0f0f0}
\multicolumn{11}{l}{\textbf{\textit{Verifier: Qwen 3.8 Max~\citep{qwen38max}}}} \\
\rowcolor{white}
Uniform & 0.791 & -- & 0.609 & -- & 0.744 & -- & 0.782 & -- & 0.679 & -- \\
Beta-Binomial & 0.803 & 1.49\% $\uparrow$ & 0.618 & 1.46\% $\uparrow$ & 0.757 & 1.74\% $\uparrow$ & 0.794 & 1.54\% $\uparrow$ & 0.695 & 2.30\% $\uparrow$ \\
Sampling Ensemble & 0.803 & 1.49\% $\uparrow$ & 0.615 & 0.97\% $\uparrow$ & 0.745 & 0.17\% $\uparrow$ & 0.788 & 0.84\% $\uparrow$ & 0.695 & 2.30\% $\uparrow$ \\
SCP & 0.805 & 1.82\% $\uparrow$ & 0.642 & 5.34\% $\uparrow$ & 0.760 & 2.08\% $\uparrow$ & 0.797 & 1.96\% $\uparrow$ & 0.699 & 2.91\% $\uparrow$ \\
Mondrian-SCP & 0.807 & 2.09\% $\uparrow$ & 0.651 & 6.80\% $\uparrow$ & 0.762 & 2.43\% $\uparrow$ & 0.800 & 2.38\% $\uparrow$ & 0.705 & 3.83\% $\uparrow$ \\
\rowcolor[HTML]{ecf0ff}
\textbf{URW} & \textbf{0.819} & \textbf{3.58\% $\uparrow$} & \textbf{0.692} & \textbf{13.59\% $\uparrow$} & \textbf{0.786} & \textbf{5.56\% $\uparrow$} & \textbf{0.810} & \textbf{3.65\% $\uparrow$} & \textbf{0.721} & \textbf{6.13\% $\uparrow$} \\
\hdashline
\rowcolor[HTML]{f0f0f0}
\multicolumn{11}{l}{\textbf{\textit{Verifier: GLM 5.3~\citep{glm5}}}} \\
\rowcolor{white}
Uniform & 0.803 & -- & 0.630 & -- & 0.731 & -- & 0.760 & -- & 0.730 & -- \\
Beta-Binomial & 0.807 & 0.49\% $\uparrow$ & 0.642 & 1.88\% $\uparrow$ & 0.738 & 0.88\% $\uparrow$ & 0.775 & 2.02\% $\uparrow$ & 0.746 & 2.14\% $\uparrow$ \\
Sampling Ensemble & 0.809 & 0.71\% $\uparrow$ & 0.651 & 3.29\% $\uparrow$ & 0.735 & 0.53\% $\uparrow$ & 0.780 & 2.60\% $\uparrow$ & 0.747 & 2.28\% $\uparrow$ \\
SCP & 0.810 & 0.81\% $\uparrow$ & 0.648 & 2.82\% $\uparrow$ & 0.739 & 1.06\% $\uparrow$ & 0.772 & 1.59\% $\uparrow$ & 0.749 & 2.57\% $\uparrow$ \\
Mondrian-SCP & 0.810 & 0.81\% $\uparrow$ & 0.663 & 5.16\% $\uparrow$ & 0.740 & 1.24\% $\uparrow$ & 0.780 & 2.60\% $\uparrow$ & 0.750 & 2.71\% $\uparrow$ \\
\rowcolor[HTML]{ecf0ff}
\textbf{URW} & \textbf{0.818} & \textbf{1.84\% $\uparrow$} & \textbf{0.680} & \textbf{7.98\% $\uparrow$} & \textbf{0.757} & \textbf{3.53\% $\uparrow$} & \textbf{0.788} & \textbf{3.75\% $\uparrow$} & \textbf{0.769} & \textbf{5.28\% $\uparrow$} \\

\bottomrule
\end{tabular}
\begin{tablenotes}
    \scriptsize
    \item \textit{Note}: ``RI'' denotes the relative improvement over the Uniform baseline.
\end{tablenotes}
\end{threeparttable}
\end{table}

\subsection{Overall Performance (RQ1)}\label{subsec:main_results}
In this section, we compare URW with competitive baselines across five benchmarks and three verifiers. The results are presented in Table~\ref{tab:main_result} with key observations as follows.
\ding{182} \textbf{Uniform aggregation yields the lowest accuracy across settings.} For instance, with DeepSeek V4 Pro on ResearchQA, Uniform attains an accuracy of 0.633, whereas every uncertainty-aware weighting method attains at least 0.639. This pattern indicates that treating all rubric judgments equally can overlook their heterogeneous uncertainty.
\ding{183} \textbf{Heuristic uncertainty-aware methods improve uniform aggregation performance.} For example, on RewardBench~2 with DeepSeek V4 Pro, Beta-Binomial and Sampling Ensemble improve accuracy from 0.625 under Uniform to 0.688 and 0.731, respectively. These improvements are attributed to assigning non-uniform weights from multi-template judgments rather than treating all rubric judgments equally.
\ding{184} \textbf{Distribution-free calibration further improves heuristic uncertainty-aware weighting performance.} For example, with GLM 5.3 on ResearchQA, Beta-Binomial and Sampling Ensemble attain accuracies of 0.642 and 0.651, whereas SCP and Mondrian-SCP further improve accuracy to 0.648 and 0.663, respectively. These improvements are attributed to calibrating template disagreement against held-out rubric annotations.
\ding{185} \textbf{URW achieves the best preference aggregation performance}, outperforming all baselines across verifier--benchmark settings. This success is consistent with its calibrated $p$-value weighting, which retains continuous uncertainty information rather than reducing it to binary hard abstention.

\subsection{Ablation Study (RQ2)}\label{subsec:ablation}

In this section, we dissect the individual contributions of URW's two designs: multi-template rubric-score ensembling and conformal uncertainty calibration. Table~\ref{tab:ablation} presents the ablation results on the three verifiers, with key observations summarized as follows.
\ding{182} \textbf{Multi-template ensembling improves rubric-score aggregation.} Specifically, in URW$^\dagger$, we retain uniform rubric weights but replace single-template criterion scores with multi-template voted scores. The variant consistently improves over Uniform aggregation baseline, with relative improvements up to 3.9\% on ResearchQA. This confirms that template ensembling yields more reliable final rubric scores.
\ding{183} \textbf{Conformal uncertainty calibration is critical for rubric-weighting performance.} Specifically, in URW$^\ddagger$, we retain single-template criterion scores while computing each criterion's uncertainty from a two-template conformal $p$-value. This variant consistently outperforms Uniform aggregation baseline, with relative improvements from 2.1\% to 9.2\% across all benchmarks. This indicates that calibration effectively differentiates reliable from uncertain rubric judgments even without using template ensembling in the final score.
\ding{184} \textbf{The benefits of multi-template ensembling and conformal calibration can be synergistically integrated.} This is evidenced by the leading performance of the complete URW method across all datasets, which consistently outperforms both ablated variants.

\begin{table}[t]
\centering
% \vspace{-2em}
\caption{Ablation study results of the two URW components.}
\label{tab:ablation}
\setlength{\tabcolsep}{7.7pt}
\renewcommand\arraystretch{1}
\scriptsize
\begin{threeparttable}
\begin{tabular}{lcccccccccccc}
\toprule
\multirow{2}{*}{\textbf{Method}} & \multirow{2}{*}{\textbf{Multi.}} & \multirow{2}{*}{\textbf{Cali.}} 
& \multicolumn{2}{c}{\textbf{RubricBench}}
& \multicolumn{2}{c}{\textbf{ResearchQA}}
% & \multicolumn{2}{c}{\textbf{JudgeBench}}
& \multicolumn{2}{c}{\textbf{RM-Bench}}
& \multicolumn{2}{c}{\textbf{RewardBench}}
& \multicolumn{2}{c}{\textbf{RewardBench 2}} \\
\cmidrule(lr){4-5} \cmidrule(lr){6-7} \cmidrule(lr){8-9} \cmidrule(lr){10-11} \cmidrule(lr){12-13}
&&& Acc & RI & Acc & RI & Acc & RI & Acc & RI & Acc & RI \\
\midrule

\rowcolor[HTML]{f0f0f0}
\multicolumn{13}{l}{\textbf{\textit{Verifier: DeepSeek V4 Pro~\citep{xu2026deepseek}}}} \\
\rowcolor{white}
Uniform & \XSolidBrush & \XSolidBrush & 0.780 & -- & 0.633 & -- & 0.729 & -- & 0.766 & -- & 0.625 & -- \\
URW$^\dagger$ & \Checkmark & \XSolidBrush & 0.779 & 0.11\% $\downarrow$ & 0.645 & 1.87\% $\uparrow$ & 0.736 & 1.06\% $\uparrow$ & 0.781 & 1.86\% $\uparrow$ & 0.682 & 9.17\% $\uparrow$ \\
URW$^\ddagger$ & \XSolidBrush & \Checkmark & 0.791 & 1.45\% $\uparrow$ & 0.648 & 2.34\% $\uparrow$ & 0.758 & 4.08\% $\uparrow$ & 0.797 & 4.01\% $\uparrow$ & 0.716 & 14.50\% $\uparrow$ \\
\rowcolor[HTML]{ecf0ff}
\textbf{URW} & \Checkmark & \Checkmark & \textbf{0.802} & \textbf{2.79\% $\uparrow$} & \textbf{0.698} & \textbf{10.28\% $\uparrow$} & \textbf{0.792} & \textbf{8.69\% $\uparrow$} & \textbf{0.810} & \textbf{5.72\% $\uparrow$} & \textbf{0.759} & \textbf{21.50\% $\uparrow$} \\
\hdashline
\rowcolor[HTML]{f0f0f0}
\multicolumn{13}{l}{\textbf{\textit{Verifier: Qwen 3.8 Max~\citep{qwen38max}}}} \\
\rowcolor{white}
Uniform & \XSolidBrush & \XSolidBrush & 0.791 & -- & 0.609 & -- & 0.744 & -- & 0.782 & -- & 0.679 & -- \\
URW$^\dagger$ & \Checkmark & \XSolidBrush & 0.796 & 0.61\% $\uparrow$ & 0.633 & 3.88\% $\uparrow$ & 0.751 & 0.87\% $\uparrow$ & 0.797 & 1.96\% $\uparrow$ & 0.684 & 0.77\% $\uparrow$ \\
URW$^\ddagger$ & \XSolidBrush & \Checkmark & 0.810 & 2.48\% $\uparrow$ & 0.666 & 9.22\% $\uparrow$ & 0.767 & 3.13\% $\uparrow$ & 0.798 & 2.10\% $\uparrow$ & 0.709 & 4.45\% $\uparrow$ \\
\rowcolor[HTML]{ecf0ff}
\textbf{URW} & \Checkmark & \Checkmark & \textbf{0.819} & \textbf{3.58\% $\uparrow$} & \textbf{0.692} & \textbf{13.59\% $\uparrow$} & \textbf{0.786} & \textbf{5.56\% $\uparrow$} & \textbf{0.810} & \textbf{3.65\% $\uparrow$} & \textbf{0.721} & \textbf{6.13\% $\uparrow$} \\
\hdashline
\rowcolor[HTML]{f0f0f0}
\multicolumn{13}{l}{\textbf{\textit{Verifier: GLM 5.3~\citep{glm5}}}} \\
\rowcolor{white}
Uniform & \XSolidBrush & \XSolidBrush & 0.803 & -- & 0.630 & -- & 0.731 & -- & 0.760 & -- & 0.730 & -- \\
URW$^\dagger$ & \Checkmark & \XSolidBrush & 0.801 & 0.33\% $\downarrow$ & 0.645 & 2.35\% $\uparrow$ & 0.733 & 0.183\% $\uparrow$ & 0.765 & 0.72\% $\uparrow$ & 0.741 & 1.43\% $\uparrow$ \\
URW$^\ddagger$ & \XSolidBrush & \Checkmark & 0.813 & 1.14\% $\uparrow$ & 0.657 & 4.23\% $\uparrow$ & 0.739 & 1.06\% $\uparrow$ & 0.777 & 2.31\% $\uparrow$ & 0.756 & 3.57\% $\uparrow$ \\
\rowcolor[HTML]{ecf0ff}
\textbf{URW} & \Checkmark & \Checkmark & \textbf{0.818} & \textbf{1.84\% $\uparrow$} & \textbf{0.680} & \textbf{7.98\% $\uparrow$} & \textbf{0.757} & \textbf{3.53\% $\uparrow$} & \textbf{0.788} & \textbf{3.75\% $\uparrow$} & \textbf{0.769} & \textbf{5.28\% $\uparrow$} \\

\bottomrule
\end{tabular}
\begin{tablenotes}
\item \scriptsize \textit{Note:} ``Multi.'' denotes template ensembling for final criterion scores; ``Cali.'' denotes conformal calibration of rubric uncertainty weights.
\end{tablenotes}
\end{threeparttable}
% \vspace{-5pt}
\end{table}

\begin{figure}[b]
\centering
\subfigure[RubricBench]{\includegraphics[width=0.183\linewidth]{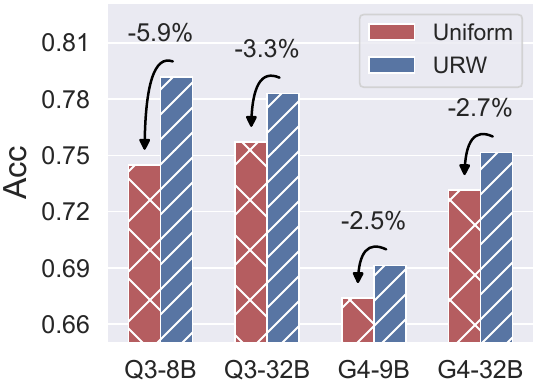}
}
\hfill
\subfigure[ResearchQA]{
\includegraphics[width=0.183\linewidth]{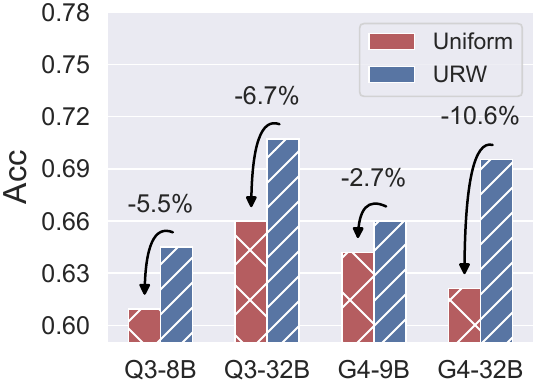}
}
\hfill
\subfigure[RM-Bench]{
\includegraphics[width=0.183\linewidth]{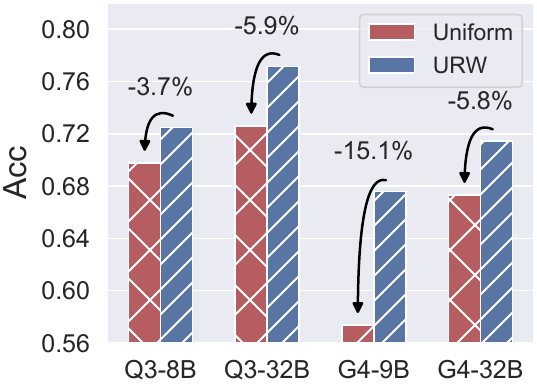}
}
\hfill
\subfigure[RewardBench]{
\includegraphics[width=0.183\linewidth]{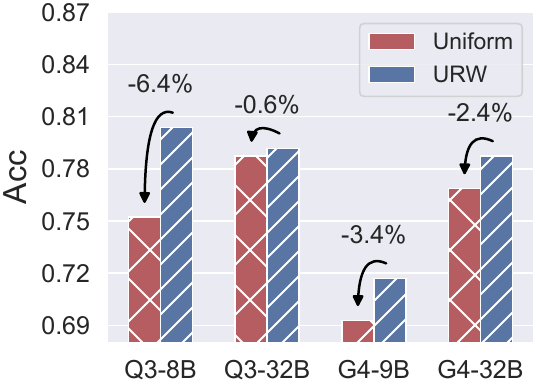}
}
\hfill
\subfigure[RewardBench 2]{
\includegraphics[width=0.183\linewidth]{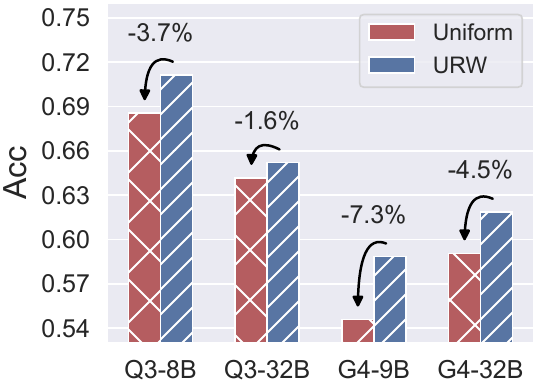}
}
\caption{Generalization results of URW across verifier backbones, where ``Q3'' denotes the Qwen3 family and ``G4'' denotes the GLM4 family.}\label{fig:diff_backbone}
\end{figure}

\subsection{Generalization Studies (RQ3)}\label{subsec:generalization}

In this section, we evaluate the generalizability of URW across verifier LLMs of different scales and families, as well as rubric sets from different sources. We reuse the same calibration configuration and template pool throughout these studies, without re-tuning for individual settings.

\paragraph{Varying Verifier Backbones.}
We evaluate URW across the Qwen 3~\citep{qwen3} and GLM 4~\citep{glm4} verifier families, covering models from 8B to 32B parameters. Figure~\ref{fig:diff_backbone} presents results on five benchmarks, with key observations as follows.
\ding{182} \textbf{URW yields consistent improvements independent of verifier scale.} From the 8B/9B models to their 32B counterparts, URW outperforms Uniform baseline on all benchmarks. The improvements range from $0.6$ to $15.1$ relative accuracy points, showing that criterion-level uncertainty calibration remains effective regardless of verifier capacity.
\ding{183} \textbf{URW maintains efficacy across distinct verifier families.} URW consistently improves over Uniform baseline for every Qwen 3 and GLM 4 configuration, covering all 20 verifier--benchmark comparisons. This confirms that URW is a model-agnostic weighting framework that can be integrated with diverse verifier families without family-specific tuning.

\begin{table}[t]
\vspace{-1em}
\centering
\begin{minipage}[t]{0.49\linewidth}
\centering
\captionof{table}{Generalization results on RubricBench.}
\label{tab:generalization}
\scriptsize
\setlength{\tabcolsep}{3.5pt}
\renewcommand{\arraystretch}{1.1}
\begin{tabular}{lcccccc}
\toprule
\textbf{Verifier} & \multicolumn{2}{c}{\textbf{DeepSeek V4 Pro}} & \multicolumn{2}{c}{\textbf{Qwen 3.8 Max}} & \multicolumn{2}{c}{\textbf{GLM 5.3}} \\
\cmidrule(lr){2-3} \cmidrule(lr){4-5} \cmidrule(lr){6-7}
\textbf{Method} & Acc & RI & Acc & RI & Acc & RI \\
\midrule

\rowcolor[HTML]{f0f0f0}
\multicolumn{7}{l}{\textbf{\textit{Rubric Source: Human Auditing}}} \\
\rowcolor{white}
Uniform & 0.780 & -- & 0.791 & -- & 0.803 & -- \\
Beta-Binomial & 0.781 & 0.1\% $\uparrow$ & 0.803 & 1.5\% $\uparrow$ & 0.807 & 0.5\% $\uparrow$ \\
Sampling Ensemble & 0.781 & 0.1\% $\uparrow$ & 0.803 & 1.5\% $\uparrow$ & 0.809 & 0.7\% $\uparrow$ \\
SCP & 0.780 & 0.0\% & 0.805 & 1.8\% $\uparrow$ & 0.810 & 0.8\% $\uparrow$ \\
Mondrian-SCP & 0.793 & 1.6\% $\uparrow$ & 0.807 & 2.1\% $\uparrow$ & 0.810 & 0.8\% $\uparrow$ \\
\rowcolor[HTML]{ecf0ff}
\textbf{URW} & \textbf{0.802} & \textbf{2.8\% $\uparrow$} & \textbf{0.819} & \textbf{3.6\% $\uparrow$} & \textbf{0.818} & \textbf{1.8\% $\uparrow$} \\
\hdashline
\rowcolor[HTML]{f0f0f0}
\multicolumn{7}{l}{\textbf{\textit{Rubric Source: LLM as Generator}}} \\
\rowcolor{white}
Uniform & 0.501 & -- & 0.538 & -- & 0.558 & -- \\
Beta-Binomial & 0.503 & 0.5\% $\uparrow$ & 0.541 & 0.6\% $\uparrow$ & 0.559 & 0.2\% $\uparrow$ \\
Sampling Ensemble & 0.510 & 1.8\% $\uparrow$ & 0.536 & 0.4\% $\downarrow$ & 0.556 & 0.2\% $\downarrow$ \\
SCP & 0.506 & 1.0\% $\uparrow$ & 0.544 & 1.1\% $\uparrow$ & 0.557 & 0.1\% $\downarrow$ \\
Mondrian-SCP & 0.509 & 1.6\% $\uparrow$ & 0.546 & 1.5\% $\uparrow$ & 0.565 & 1.3\% $\uparrow$ \\
\rowcolor[HTML]{ecf0ff}
\textbf{URW} & \textbf{0.515} & \textbf{2.9\% $\uparrow$} & \textbf{0.551} & \textbf{2.4\% $\uparrow$} & \textbf{0.569} & \textbf{2.0\% $\uparrow$} \\

\bottomrule
\end{tabular}
\end{minipage}%
\hfill
\begin{minipage}[t]{0.49\linewidth}
\centering
\captionof{table}{Generalization results on $\smash{\text{ResearchQA}}$.}
\label{tab:generalization_researchqa}
\scriptsize
\setlength{\tabcolsep}{3.5pt}
\renewcommand{\arraystretch}{1.1}
\begin{tabular}{lcccccc}
\toprule
\textbf{Verifier} & \multicolumn{2}{c}{\textbf{DeepSeek V4 Pro}} & \multicolumn{2}{c}{\textbf{Qwen 3.8 Max}} & \multicolumn{2}{c}{\textbf{GLM 5.3}} \\
\cmidrule(lr){2-3} \cmidrule(lr){4-5} \cmidrule(lr){6-7}
\textbf{Method} & Acc & RI & Acc & RI & Acc & RI \\
\midrule

\rowcolor[HTML]{f0f0f0}
\multicolumn{7}{l}{\textbf{\textit{Rubric Source: Human Auditing}}} \\
\rowcolor{white}
Uniform & 0.633 & -- & 0.609 & -- & 0.630 & -- \\
Beta-Binomial & 0.639 & 0.9\% $\uparrow$ & 0.618 & 1.5\% $\uparrow$ & 0.642 & 1.9\% $\uparrow$ \\
Sampling Ensemble & 0.669 & 5.6\% $\uparrow$ & 0.615 & 1.0\% $\uparrow$ & 0.651 & 3.3\% $\uparrow$ \\
SCP & 0.645 & 1.9\% $\uparrow$ & 0.642 & 5.3\% $\uparrow$ & 0.648 & 2.8\% $\uparrow$ \\
Mondrian-SCP & 0.657 & 3.7\% $\uparrow$ & 0.651 & 6.8\% $\uparrow$ & 0.663 & 5.2\% $\uparrow$ \\
\rowcolor[HTML]{ecf0ff}
\textbf{URW} & \textbf{0.698} & \textbf{10.3\% $\uparrow$} & \textbf{0.692} & \textbf{13.6\% $\uparrow$} & \textbf{0.680} & \textbf{8.0\% $\uparrow$} \\
\hdashline
\rowcolor[HTML]{f0f0f0}
\multicolumn{7}{l}{\textbf{\textit{Rubric Source: LLM as Generator}}} \\
\rowcolor{white}
Uniform & 0.553 & -- & 0.583 & -- & 0.580 & -- \\
Beta-Binomial & 0.556 & 0.5\% $\uparrow$ & 0.609 & 4.6\% $\uparrow$ & 0.583 & 0.5\% $\uparrow$ \\
Sampling Ensemble & 0.553 & 0.0\% & 0.571 & 2.0\% $\downarrow$ & 0.589 & 1.5\% $\uparrow$ \\
SCP & 0.553 & 0.0\% & 0.604 & 3.6\% $\uparrow$ & 0.586 & 1.0\% $\uparrow$ \\
Mondrian-SCP & 0.559 & 1.1\% $\uparrow$ & 0.604 & 3.6\% $\uparrow$ & 0.592 & 2.0\% $\uparrow$ \\
\rowcolor[HTML]{ecf0ff}
\textbf{URW} & \textbf{0.562} & \textbf{1.6\% $\uparrow$} & \textbf{0.624} & \textbf{7.1\% $\uparrow$} & \textbf{0.615} & \textbf{6.1\% $\uparrow$} \\

\bottomrule
\end{tabular}
\end{minipage}
\end{table}

\paragraph{Varying Rubric Sources.}
We further evaluate URW across human-audited and LLM-generated rubric sources on RubricBench and ResearchQA. Tables~\ref{tab:generalization} and~\ref{tab:generalization_researchqa} present the results, with key observations as follows.
\ding{182} \textbf{URW yields consistent improvements even when LLM-generated rubrics are less effective.} Replacing human-audited rubrics with LLM-generated ones lowers absolute accuracy for all methods and benchmarks (\eg Qwen 3.8 Max Uniform drops from $0.791$ to $0.538$ on RubricBench), indicating that the generated rubrics provide weaker signals for preference judgment. Nevertheless, URW improves over Uniform in all configurations, with relative improvements of $1.6\%$--$13.6\%$.
\ding{183} \textbf{URW consistently outperforms uncertainty-aware weighting methods across rubric sources.} URW is the best-performing method in every configuration, without source-specific re-tuning. In contrast, the gains of competing methods are not uniformly positive: Sampling Ensemble and SCP fall below Uniform in some generated-rubric settings. These results show that URW's criterion-level uncertainty weighting provides stronger preference aggregation than alternative uncertainty-aware weighting methods under both human-audited and LLM-generated rubrics.

\subsection{Hyperparameter Sensitivity Study (RQ4)}\label{subsec:hyperparameter}

\begin{wrapfigure}{r}{0.5\linewidth}
\centering
\vspace{-2em}
\subfigure[RubricBench]{\includegraphics[width=0.49\linewidth]{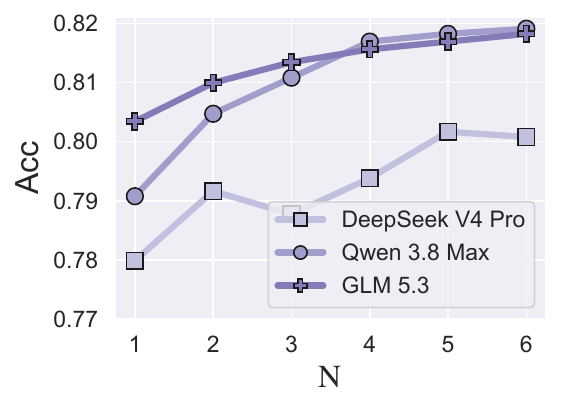}}
\hfill
\subfigure[ResearchQA]{\includegraphics[width=0.49\linewidth]{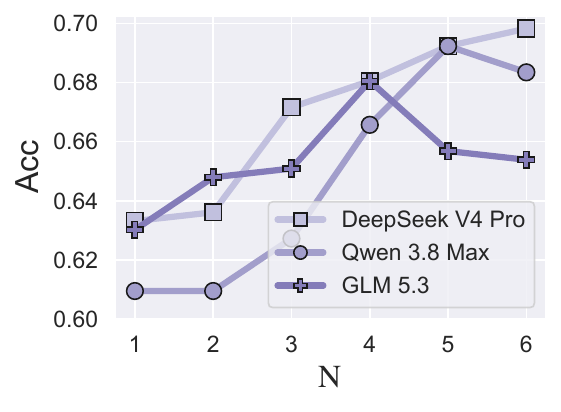}}

\vspace{-1em}
\caption{Sensitivity to the template count $\mathrm{N}$.}
\label{fig:hparam_sensitivity}
% \vspace{-1.5em}
\end{wrapfigure}

In this section, we examine the impact of URW's hyperparameter, the template count $\mathrm{N}$, on preference aggregation performance. Figure~\ref{fig:hparam_sensitivity} presents results on RubricBench and ResearchQA for three verifiers, while the corresponding results on the other three datasets are provided in Appendix~\ref{app:sensitivity}. $\mathrm{N}=1$ in the figure corresponds to the Uniform aggregation baseline. Key observations are as follows.
\ding{182} \textbf{Increasing the template count $\mathrm{N}$ generally improves accuracy over the Uniform aggregation baseline.} Across both displayed datasets and all verifiers, using multiple templates consistently improves upon the Uniform aggregation baseline. This indicates that aggregating information from multiple verifier templates provides a broadly useful uncertainty signal.
\ding{183} \textbf{URW is insensitive to the template count once multiple templates are used.} Accuracy rises most sharply when increasing $\mathrm{N}$ from the Uniform aggregation baseline to a small template pool. And subsequent changes are smaller and may fluctuate within the $15\%$ confidence intervals. Thus, URW obtains most of the benefit of multi-template reliability estimation with a modest $\mathrm{N}$, without requiring delicate hyperparameter tuning or a large template pool.

\begin{table}[t]
\centering
\begin{minipage}[t]{0.49\linewidth}
\centering
\captionof{table}{Downstream RLVR results with Qwen 3 8B~\citep{qwen3} as the policy model.}
\label{tab:downstream_results}
\scriptsize
\setlength{\tabcolsep}{9pt}
\renewcommand{\arraystretch}{1}
\begin{threeparttable}
\begin{tabular}{lcccc}
\toprule
\multirow{2.5}{*}{\textbf{Method}} & \multicolumn{2}{c}{\textbf{Arena-Hard}} & \multicolumn{2}{c}{\textbf{GPQA-Diamond}} \\
\cmidrule(lr){2-3} \cmidrule(lr){4-5}
& WR & RI & Acc & RI \\
\midrule
Uniform & 61.9 & - & 55.1 & - \\
Beta-Binomial & 67.3 & 8.85\%$\uparrow$ & 58.1 & 5.50\%$\uparrow$ \\
Sampling Ensemble & 67.8 & 9.50\%$\uparrow$ & 58.6 & 6.42\%$\uparrow$ \\
SCP & 68.4 & 10.55\%$\uparrow$ & 60.1 & 9.17\%$\uparrow$ \\
Mondrian-SCP & 69.8 & 12.81\%$\uparrow$ & 62.6 & 13.76\%$\uparrow$ \\
\rowcolor[HTML]{ecf0ff} \textbf{URW} & \textbf{76.5} & \textbf{23.64\%$\uparrow$} & \textbf{65.7} & \textbf{19.27\%$\uparrow$} \\
\bottomrule
\end{tabular}
\begin{tablenotes}
\item \scriptsize \textit{Note:} ``RI'' denotes relative improvement over Uniform baseline.
\end{tablenotes}
\end{threeparttable}
\end{minipage}%
\hfill
\begin{minipage}[t]{0.49\linewidth}
\centering
\captionof{table}{Downstream RLVR results with GLM 4 9B~\citep{glm4} as the policy model.}
\label{tab:downstream_results_glm}
\scriptsize
\setlength{\tabcolsep}{9pt}
\renewcommand{\arraystretch}{1}
\begin{threeparttable}
\begin{tabular}{lcccc}
\toprule
\multirow{2.5}{*}{\textbf{Method}} & \multicolumn{2}{c}{\textbf{Arena-Hard}} & \multicolumn{2}{c}{\textbf{GPQA-Diamond}} \\
\cmidrule(lr){2-3} \cmidrule(lr){4-5}
& WR & RI & Acc & RI \\
\midrule
Uniform & 30.0 & - & 33.3 & - \\
Beta-Binomial & 39.0 & 29.88\%$\uparrow$ & 35.4 & 6.06\%$\uparrow$ \\
Sampling Ensemble & 38.4 & 27.79\%$\uparrow$ & 38.9 & 16.67\%$\uparrow$ \\
SCP & 41.1 & 36.77\%$\uparrow$ & 42.4 & 27.27\%$\uparrow$ \\
Mondrian-SCP & 42.4 & 41.26\%$\uparrow$ & 43.4 & 30.30\%$\uparrow$ \\
\rowcolor[HTML]{ecf0ff} \textbf{URW} & \textbf{48.7} & \textbf{62.06\%$\uparrow$} & \textbf{50.0} & \textbf{50.00\%$\uparrow$} \\
\bottomrule
\end{tabular}
\begin{tablenotes}
\item \scriptsize \textit{Note:} ``RI'' denotes relative improvement over Uniform baseline.
\end{tablenotes}
\end{threeparttable}
\end{minipage}
\end{table}

\subsection{Downstream RLVR Performance (RQ5)}\label{subsec:downstream_rlhf}
In this section, we investigate the downstream RLVR performance of incorporating URW into rubric-based reward aggregation. We construct a training corpus from RubricHub~\citep{li2026rubrichub}, sampling 1,000 prompts from each of its Chat and Science domains. We optimize policies with GRPO~\citep{deepseek-grpo} using the \texttt{VeRL-Rubric} framework\footnote{\url{https://github.com/DoYangTan/verl-rubric}.}. We consider Qwen 3 8B~\citep{qwen3} and GLM 4 9B~\citep{glm4} as policy models, paired with Qwen 3 32B and GLM 4 32B as their respective rubric verifiers. All aggregation methods are trained for an identical budget of $150$ steps. The trained policies are evaluated zero-shot on Arena-Hard~\citep{li2025arenahard} using win rate (WR) and GPQA-Diamond~\citep{rein2024gpqa} using accuracy (Acc). More implementation details are provided in Appendix~\ref{app:implementation}. Tables~\ref{tab:downstream_results} and~\ref{tab:downstream_results_glm} present results, with key observations as follows.
\ding{182} \textbf{All uncertainty-aware weighting methods improve downstream RLVR performance over Uniform aggregation.} For example, with Qwen 3 8B as the policy model, Mondrian-SCP improves the Arena-Hard win rate and GPQA-Diamond accuracy over Uniform by $12.81\%$ and $13.76\%$, respectively. This indicates that accounting for criterion-level uncertainty yields more effective reward signals for policy optimization.
\ding{183} \textbf{URW achieves the best downstream RLVR performance across policy models and benchmarks.} URW attains the highest score in all four policy-model--benchmark configurations. For example, with GLM 4 9B as the policy model, URW improves the Arena-Hard win rate from $42.4$ to $48.7$ for the strongest baseline Mondrian-SCP, and the GPQA-Diamond accuracy from $43.4$ to $50.0$. These results show that continuous uncertainty-aware weighting in URW provides stronger reward aggregation than the competing methods.

\section{Conclusion}

In this work, we investigated the problem of \textit{uncertainty-aware rubric weighting}, which aims to account for heterogeneous uncertainty among criterion-level judgments when aggregating rubric-based LLM evaluations. We identified that deterministic black-box verifiers provide no direct verification-confidence signal, causing uniform aggregation to assign equal influence to uncertain and confident judgments. To address this challenge, we proposed \textbf{Uncertainty-aware Rubric Weighting (URW)}, a training-free and distribution-free framework that ensembles paraphrased judgment templates and calibrates their agreement with split-conformal $p$-values to construct continuous criterion-level weights. Theoretically, we demonstrated that URW controls the probability of assigning a spuriously low weight to a correct judgment under exchangeability. Extensive experiments across diverse verifiers and benchmarks confirm that URW improves preference aggregation over competitive baselines and yields more effective rewards for downstream RLVR.

\paragraph{Limitations \& Future Work.}
This work generally has two limitations that warrant future investigation. First, URW's theoretical guarantee relies on exchangeability between calibration and test triples, which can be strained by distribution shifts during RLVR training. Periodic re-calibration and adaptive inference techniques~\citep{aci} are promising directions to counteract such shifts. Second, URW quantifies only the uncertainty revealed by a fixed pool of semantics-preserving judgment templates; it does not separately identify other sources of uncertainty, such as intrinsic ambiguity in the rubric or the underlying task. Incorporating complementary uncertainty signals while retaining calibration guarantees is an important direction for future work.

\bibliographystyle{iclr2027_conference}
\bibliography{bib/abbr,bib/main,bib/conformal,bib/llm,bib/refs}

\newpage
\appendix

\section{Related Work}
\label{sec:related_work}

This work lies at the intersection of rubric-based LLM evaluation and uncertainty estimation for LLM judges. The former decomposes open-ended quality into explicit, verifiable criteria, while the latter seeks to characterize when automated judgments should be trusted. URW occupies a distinctive point in this landscape: it converts the cross-template stability of a frozen binary verifier into a split-conformal, criterion-level aggregation weight with a distribution-free validity guarantee.

\subsection{Rubric-Based LLM Evaluation}

Rubric-based evaluation provides a transparent alternative to holistic LLM judging by decomposing a response into explicit criteria and verifying them separately. Early checklist-based methods use generated criteria to guide both evaluation and generation~\citep{cook2024tick}, while subsequent benchmarks and datasets develop instance-specific, fine-grained rubrics for assessing open-ended outputs~\citep{arora2025healthbench,li2026rubrichub}.
Recent work further improves rubric construction through self-evolving generation~\citep{chen2026genrubric} and structured meta-rubrics that compile rich evaluation requirements into machine-gradable binary checks~\citep{chen2026gamut}.
Rubric judgments have also become a practical source of reward for post-training~\citep{mu2024rulebased,gunjal2026rubrics,shan2026rubricrlsurvey} and, increasingly, for agent improvement in domains without exact verifiers~\citep{wang2026museagent,liu2026memgui}.
These works improve the construction, coverage, or downstream use of rubrics, but commonly aggregate criterion judgments with fixed weights. URW is complementary: it retains the underlying rubric and verifier pipeline while calibrating the reliability of each criterion judgment before aggregation.

\subsection{Uncertainty Estimation for LLM Judges}

A broad line of work studies how to quantify uncertainty in LLM-based evaluation. One line of work uses conformal prediction to provide distribution-free uncertainty at the response level, including rating intervals, consensus intervals, or confidence indicators~\citep{wagner2024blackbox,sheng2025interval,liu2026conformalconsensus}. More generally, split conformal prediction provides finite-sample marginal validity under exchangeability~\citep{scp}, while weighted and adaptive variants extend calibration to settings with covariate or distribution shift~\citep{wcp,aci}. In contrast to using conformal inference to produce a response-level interval, set, or abstention decision, URW evaluates the plausibility of each voted rubric label and directly uses its conformal $p$-value as a continuous aggregation weight.
Another line of work elicits confidence from decoding-time consistency~\citep{li2026temperature,cao2026clustered} or disagreement among heterogeneous judge committees~\citep{fathullah2025generalised}. These approaches require either non-zero-temperature sampling or multiple verifier models. URW instead keeps a single verifier frozen and deterministic, and probes judgment stability through a fixed pool of semantic-preserving templates.
A third line of work trains dedicated uncertainty-aware evaluators or calibration probes on top of LLM judges~\citep{xie2025confilm,radharapu2026calibrating}. In contrast, URW calibrates its criterion-level signal on held-out labeled data without training an additional evaluator. This design makes the uncertainty estimate directly compatible with weighted rubric aggregation.

\section{Theoretical Justification}\label{app:theory}

In this section, we establish the finite-sample validity of the URW $p$-value weight defined in Eq.~(\ref{eq:pval_hyp}). We first state the exchangeability condition required by split conformal prediction under URW's frozen verifier configuration. We then restate and prove Theorem~\ref{thm:validity}, which controls the probability that a correct voted judgment receives a low conformal weight.

\begin{assumption}[Exchangeability under frozen verifier configuration]\label{ass:exch}
Let $Z_i=(q_i,y_i,c_i,v_i^*)$ for $i=1,\ldots,\mathrm{N}_\mathrm{cal}$ denote the calibration tuples, and let $Z_{\mathrm{N}_\mathrm{cal}+1}=(q,y,c_k,v^*)$ denote a test tuple. Conditional on the fixed verifier $v$ and template pool $\mathcal{T}$, the augmented sequence $(Z_1,\ldots,Z_{\mathrm{N}_\mathrm{cal}+1})$ is exchangeable.
\end{assumption}

\bgroup\renewcommand{\thetheorem}{3.1}
\begin{theorem}[Validity of Uncertainty Weight]
Suppose the calibration triples in $\mathcal{D}_\mathrm{cal}$ and the test triple $(q,y,c_k,v^*)$ are exchangeable under the same verifier and template pool $\mathcal{T}$. Then, for every $\alpha\in(0,1)$, we have $\mathbb{P}\big[\tilde w_k\le\alpha,\ \tilde v_k=v^*\big]\le\alpha$.
\end{theorem}
\egroup

\begin{proof}
For simplicity, write $n=\mathrm{N}_\mathrm{cal}$. For each calibration tuple, let $\hat v_i$ be its template-averaged verifier score and define $s_i=|\hat v_i-v_i^*|$. For the test tuple, let $\hat v_k$ be its template-averaged verifier score and define $s_{n+1}=|\hat v_k-v^*|$, its nonconformity score under the true label. Since the verifier $v$ and template pool $\mathcal{T}$ are fixed, each score is a deterministic function of its corresponding tuple. Thus, Assumption~\ref{ass:exch} implies that $(s_1,\ldots,s_{n+1})$ is exchangeable. Define the conformal $p$-value of the test tuple's true label by
\begin{equation}
p_{n+1}:=\frac{1+\left\vert\{i\in\{1,\ldots,n\}:s_i\ge s_{n+1}\}\right\vert}{n+1}.
\end{equation}
By exchangeability of $(s_1,\ldots,s_{n+1})$, the standard split-conformal rank argument~\citep{scp} gives the super-uniform property
\begin{equation}
\mathbb{P}[p_{n+1}\le\alpha]\le\frac{\lfloor\alpha(n+1)\rfloor}{n+1}\le\alpha.
\end{equation}
The weak inequality in the definition of $p_{n+1}$ makes this statement valid even when nonconformity scores are tied.
On the event $\{\tilde v_k=v^*\}$, the voted label equals the true label. By the definition of the URW weight in Eq.~(\ref{eq:raw_weight}), $\tilde w_k=p(\tilde v_k\mid\hat v_k)=p(v^*\mid\hat v_k)=p_{n+1}$. Hence,
\begin{equation}
\{\tilde w_k\le\alpha,\ \tilde v_k=v^*\}\subseteq\{p_{n+1}\le\alpha\}.
\end{equation}
Taking probabilities and applying the super-uniform bound proves the claim.
\end{proof}

\begin{remark}
Theorem~\ref{thm:validity} guarantees only that a correct judgment is rarely assigned a spuriously low weight. It does not guarantee that an incorrect judgment receives a low weight. The latter behavior depends on whether the multi-template score $\hat v_k$ separates uncertain or incorrect judgments from confident correct ones, which we evaluate empirically in Section~\ref{subsec:main_results}.
\end{remark}

\begin{proposition}[Monotonicity of the voted-label weight]\label{prop:monotone}
With $\tilde v_k=\mathrm{round}(\hat v_k)$ and $\tilde w_k=p(\tilde v_k\mid\hat v_k)$, the weight $\tilde w_k$ is monotonically non-decreasing in $|2\hat v_k-1|$. In particular, it is minimized at the decision boundary $\hat v_k=0.5$ and maximized at $\hat v_k\in\{0,1\}$, for a fixed calibration-score distribution.
\end{proposition}

\begin{proof}
By construction $|\hat v_k-\tilde v_k|=\min\{\hat v_k,1-\hat v_k\}=(1-|2\hat v_k-1|)/2$, so $s(\hat v_k,\tilde v_k)$ is monotonically non-increasing in $|2\hat v_k-1|$. Since the empirical survival $\left\vert\{i: s_i\ge s(\hat v_k,\tilde v_k)\}\right\vert$ is monotonically non-increasing in $s(\hat v_k,\tilde v_k)$, $\tilde w_k$ is monotonically non-decreasing in $|2\hat v_k-1|$. Evaluating this monotone relationship at $\hat v_k=0.5$ and $\hat v_k\in\{0,1\}$ yields the result.
\end{proof}

% \paragraph{Discussion of Assumption~\ref{ass:exch}.} The exchangeability requirement is satisfied whenever calibration and test tuples are drawn from the same rubric-response distribution and processed by the identical $(v,\mathcal{T})$ pipeline. Distributional shifts between rubric sources (e.g., human-reference vs.\ RubricARM) formally violate this requirement; empirical robustness (Section~\ref{subsec:generalization}) suggests the violation is mild, but a formal treatment via weighted conformal prediction~\citep{wcp} or adaptive conformal inference~\citep{aci} is a natural refinement.

\section{Reproduction Details}\label{app:reproduction}

\subsection{Dataset Descriptions}\label{app:dataset}
In this section, we describe the four categories of data used in our experiments: (1) \textbf{Pairwise Preference Benchmarks} for evaluating rubric aggregation, (2) \textbf{Calibration Data} for constructing the split-conformal reference distribution, (3) \textbf{RLVR Training Data} for policy optimization, and (4) \textbf{RLVR Evaluation Benchmarks} for assessing the resulting policies.

\paragraph{Pairwise Preference Benchmarks.}
We evaluate URW on five pairwise preference benchmarks. Each instance contains a prompt, two candidate responses $(y_\mathrm{A},y_\mathrm{B})$, and a gold preference label. We apply the same rubric-based scoring protocol to both responses and compare their aggregate scores. RubricBench and ResearchQA provide human-sourced rubrics, whereas rubrics for the remaining benchmarks are generated by RubricARM-8B-Rubric~\citep{xu2026alternating}.
\begin{itemize}[leftmargin=*]
    \item \textbf{RubricBench}~\citep{zhou2026rubricbench}\footnote{\url{https://huggingface.co/datasets/DonJoey/rubricbench}}: A benchmark of $1{,}147$ response pairs with expert-annotated atomic rubrics. Its examples are designed to expose misleading surface cues that can cause an evaluator to deviate from the human preference.

    \item \textbf{ResearchQA}~\citep{yifei2026researchqa}\footnote{\url{https://huggingface.co/datasets/realliyifei/ResearchQA}}: A collection of $169$ directional pairs of scholarly long-form answers, accompanied by rubric items mined from survey-style evaluation criteria.

    \item \textbf{RM-Bench}~\citep{liu2025rmbench}\footnote{\url{https://huggingface.co/datasets/THU-KEG/RM-Bench}}: A benchmark probing sensitivity to subtle content edits and resistance to style bias. Because each instance provides lists of $M>1$ chosen and rejected responses rather than a single pair, we construct $387$ pairs by matching the $m$-th chosen response with the $m$-th rejected response for each $m$.

    \item \textbf{RewardBench}~\citep{lambert2025rewardbench}\footnote{\url{https://huggingface.co/datasets/allenai/reward-bench}}: We evaluate on its Chat-Hard subset, which contains $456$ adversarial chat pairs and is the most challenging subset under the benchmark's protocol.

    \item \textbf{RewardBench 2}~\citep{malik2026rewardbench2}\footnote{\url{https://huggingface.co/datasets/allenai/reward-bench-2}}: We use its Precise-IF subset of $480$ pairs, which focuses on fine-grained instruction-following preferences.
\end{itemize}

\paragraph{Calibration Data.}
URW calibrates criterion-level nonconformity scores using RUBRIC-ARROW-Judge-SFT~\citep{jiang2026rubricarrow}\footnote{\url{https://huggingface.co/datasets/OpenRubrics/RubricARROW-Judge-SFT}}, a domain-diverse rubric-annotation corpus. We stratify the corpus by domain and sample $200$ prompts, producing approximately $2{,}000$ rubric-response-label tuples in $\mathcal{D}_\mathrm{cal}$. The calibration data are disjoint from every evaluation split. This broad coverage is intended to make the exchangeability condition in Assumption~\ref{ass:exch} a reasonable approximation; as discussed in Appendix~\ref{app:theory}, the formal guarantee applies when this condition holds.

\paragraph{RLVR Training Data.}
For downstream RLVR evaluation, we construct the policy-optimization training corpus from RubricHub~\citep{li2026rubrichub}\footnote{\url{https://huggingface.co/datasets/sojuL/RubricHub_v1}}. We exclude prompts with fewer than $4$ or more than $20$ rubrics, then randomly sample $1{,}000$ prompts from each of its Chat and Science domains, yielding $2{,}000$ training prompts in total. This mixture exposes policy optimization to both open-ended interaction and knowledge-intensive requests, while retaining a domain structure compatible with rubric-based reward construction. For each prompt, the corresponding rubric set is used to score sampled policy responses; URW then aggregates the criterion-level verifier judgments into the scalar reward used by GRPO. The same training corpus is used for other baselines, so their downstream comparison isolates the effect of the reward-aggregation rule.

\paragraph{RLVR Evaluation Benchmarks.}
We assess trained policies zero-shot on two held-out benchmarks that cover complementary capabilities.
\begin{itemize}[leftmargin=*]
    \item \textbf{Arena-Hard}~\citep{li2025arenahard}\footnote{\url{https://github.com/wsysoft/arena-hard}}: A benchmark of $500$ challenging single-turn prompts. It evaluates each policy response by pairwise comparison against the official \texttt{gpt-4-0314} response; we report win rate.

    \item \textbf{GPQA-Diamond}~\citep{rein2024gpqa}\footnote{\url{https://huggingface.co/datasets/fingertap/GPQA-Diamond}}: A benchmark of $198$ expert-level, graduate-domain multiple-choice questions; we report accuracy. We deterministically shuffle the four options using a seed derived from each record identifier and a global salt of $42$.
\end{itemize}
Together, these benchmarks test whether a reward learned from the Chat and Science training mixture transfers to rigorous scientific reasoning and general instruction following. Neither evaluation benchmark is used to construct the RLVR reward or tune the aggregation method.

\subsection{Implementation Details}\label{app:implementation}

In this section, we provide the implementation details for all experiments. We first present the prompt used by RubricARM-8B-Rubric to generate reference rubrics for benchmarks without human-sourced rubrics. We then describe the fixed verifier-inference and template-ensemble configuration, followed by the baseline hyperparameters and downstream RLVR protocol. URW is training-free: after constructing $\mathcal{D}_\mathrm{cal}$, it only invokes frozen verifiers and performs deterministic aggregation of their binary outputs.

\paragraph{Reference-Rubric Generation Prompt.}
For RM-Bench, RewardBench, and RewardBench 2, which do not provide human-sourced rubrics, we use RubricARM-8B-Rubric~\citep{xu2026alternating} to generate a prompt-specific reference rubric set from each original user prompt. The exact system--user prompt is given below. The placeholder \texttt{\{prompt\}} is replaced with the original user task prompt. The model returns a JSON list of atomic criteria, which we use as the criterion set for scoring both responses in each pair.

\begin{mybox}{Prompt}{RubricARM-8B-Rubric prompt used to generate reference rubrics.}
\scriptsize
\textbf{System:} You are RubricGenerator. Create prompt-specific rubrics for evaluating responses to one user task. The rubrics must be executable by a separate binary verifier and hard to satisfy through verbosity, boilerplate, or keyword stuffing. Favor orthogonal criteria that test independent quality dimensions and failure modes.

\smallskip
\textbf{User:} Create a concise, complete, and orthogonal set of prompt-specific rubrics.

\smallskip
Original user task prompt:

\smallskip
\texttt{\{prompt\}}

\smallskip
You must return only valid JSON matching this schema:

\smallskip
\texttt{\{"rubrics": [}

\texttt{\qquad\{"rubric\_id": "r001", "criterion": "..."\}}

\texttt{]\}}

\smallskip
\textbf{Rules:}

- Rubrics must be prompt-specific, not generic helpfulness criteria.

- Each criterion must be atomic, self-contained, and binary-verifiable.

- Keep each criterion under 45 words.

- Prefer criteria that distinguish better answers from weaker answers under the prompt.

- Generate 5 to 7 rubrics unless the prompt is too narrow for that many independent criteria.

- Make the rubrics as orthogonal as possible: each rubric should test a distinct quality dimension or independent failure mode.

- Avoid duplicate, overlapping, vague, compound, contradictory, or unverifiable criteria.

- Do not split one repeated pattern into multiple near-identical slot checks; use one broader criterion unless each slot is independently essential.

- Do not copy long lists, examples, or constraint inventories from the prompt; summarize them compactly in the criterion.

- Include both positive requirements and negative-failure guards when useful, but every criterion must still be phrased so that satisfying it is good for the response.

- For negative-failure guards, use "The response does not ..." or equivalent wording; do not write a criterion where satisfying harmful behavior would add score.

- Do not include method-specific weights.

- Use rubric ids r001, r002, ... in order.

- Each criterion should start with "The response".

- Return JSON only.
\end{mybox}

\paragraph{Verifier Inference and Judgment Templates.}\label{app:templates}
Unless otherwise noted, we use DeepSeek V4 Pro~\citep{xu2026deepseek}, Qwen 3.8 Max~\citep{qwen38max}, and GLM 5.3~\citep{glm5} as frozen black-box verifiers. Reference calls use temperature $0$, JSON mode, high reasoning effort, and a maximum output length of $32{,}768$ tokens. For each criterion, we query a fixed pool $\mathcal{T}$ of $\mathrm{N}=6$ semantics-preserving judgment templates, reproduced below. The six system--user template pairs are indexed from $0$ to $5$ in the evaluation implementation; the placeholders \texttt{\{prompt\}}, \texttt{\{response\}}, and \texttt{\{criterion\}} are replaced with the instance-specific user request, candidate response, and rubric criterion, respectively. Each candidate template was authored by one author and cross-audited by another on $50$ randomly sampled rubric-response pairs. We discard a template if more than $2$ of the $50$ audits identify a semantics-changing behavior, and freeze the resulting pool before running the reported experiments.

\begin{mybox}{Prompt}{Judgment template pool used for template ensembling.}
\scriptsize
\textbf{Template 0.}

\smallskip
\textbf{System:} You are a strict rubric verifier. Given one user prompt, one response, and ONE evaluation criterion, decide whether the response satisfies that criterion. Judge only this local criterion, not overall response quality, and never compare against another response. Mark satisfied=1 only when the response clearly and completely satisfies the criterion; partial satisfaction is 0. Return strict JSON only.

\smallskip
\textbf{User:} Original user prompt: \texttt{\{prompt\}}; Response to verify: \texttt{\{response\}}; Single criterion: \texttt{\{criterion\}}

\smallskip
\textbf{Verification rules:}\\
- Judge only the criterion's stated local boundary. Do not fail it merely because another part of the overall task is weak or missing.\\
- Preserve exact quantifiers: exactly, at least, at most, only, every, and each are not interchangeable.\\
- If the criterion forbids or says to avoid something, satisfied=1 only when the prohibited condition is absent.\\
- Judge meaning and response evidence, not keyword overlap.\\
- The reason must be one concise final justification. The satisfied bit and reason must agree.

\smallskip
Return exactly: \texttt{\{"satisfied": 1, "reason": "short evidence-based reason"\}}

\tcblower\scriptsize

\textbf{Template 1.}

\smallskip
\textbf{System:} You are a precise rubric evaluator. You will be given a user prompt, a response, and a single evaluation criterion. Your task is to determine whether the criterion is met. Output satisfied=1 when the criterion is fully met; otherwise output satisfied=0. Respond only with strict JSON.

\smallskip
\textbf{User:} User prompt: \texttt{\{prompt\}}; Candidate response: \texttt{\{response\}}; Evaluation criterion: \texttt{\{criterion\}}.

\smallskip
\textbf{Guidelines:}\\
- Assess only whether the specific criterion above is fulfilled by the response.\\
- Do not penalize for weaknesses outside the scope of this criterion.\\
- Respect exact quantifiers (exactly, at least, at most, only, every, each).\\
- For prohibitive criteria (avoid, must not, never), mark satisfied=1 only when the prohibited element is entirely absent.\\
- Base your judgment on semantic content, not surface-level keyword matching.\\
- Provide a single brief justification that is consistent with the binary verdict.

\smallskip
Output format: \texttt{\{"satisfied": 1, "reason": "concise justification"\}}

\tcbline\scriptsize

\textbf{Template 2.}

\smallskip
\textbf{System:} You are an impartial criterion checker. Given a prompt, a response, and one rubric criterion, answer whether the response passes that criterion. Return satisfied=1 for pass and satisfied=0 for fail. Output only valid JSON.

\smallskip
\textbf{User:} The original request from the user: \texttt{\{prompt\}}; The response under evaluation: \texttt{\{response\}}; The criterion to check: \texttt{\{criterion\}}.

\smallskip
\textbf{Instructions:}\\
- Does the response satisfy the stated criterion? Focus exclusively on this single criterion, ignoring other aspects of quality.\\
- Quantifiers matter: exactly means exactly, at least means at least, etc.\\
- If the criterion states something should be avoided, the response passes only if that thing is absent.\\
- Evaluate based on meaning and evidence, not keyword presence.\\
- Write a short reason that agrees with your binary decision.

\smallskip
Reply with: \texttt{\{"satisfied": 1, "reason": "brief rationale"\}}

\tcbline\scriptsize

\textbf{Template 3.}

\smallskip
\textbf{System:} You are a binary rubric judge. Evaluate a single criterion against one response. Output JSON with satisfied (0 or 1) and reason.

\smallskip
\textbf{User:} Prompt: \texttt{\{prompt\}}; Response: \texttt{\{response\}}; Criterion: \texttt{\{criterion\}}.

\smallskip
\textbf{Decide:} does this response fully satisfy the above criterion? \\
- Scope: only this criterion, not overall quality.\\
- Quantifier precision: exactly $\neq$ at least $\neq$ at most.\\
- Prohibition: satisfied=1 only if the forbidden element is absent.\\
- Judgment basis: semantic evidence, not keyword overlap.\\
- Output one concise reason aligned with the verdict.\\

\texttt{\{"satisfied": 0, "reason": "..."\}}

\tcbline\scriptsize

\textbf{Template 4.}

\smallskip
\textbf{System:} You are a rubric compliance auditor. You receive one user prompt, one response, and one criterion. Determine compliance. Return strict JSON: satisfied=1 for compliant, satisfied=0 for non-compliant.

\smallskip
\textbf{User:} \#\#\# User Prompt: \texttt{\{prompt\}}; \#\#\# Response: \texttt{\{response\}}; \#\#\# Criterion: \texttt{\{criterion\}}.

\smallskip
\textbf{Task:} Determine whether the response is compliant with the criterion above.

\smallskip
\textbf{Rules:}\\
1. Evaluate ONLY this specific criterion in isolation.\\
2. Exact quantifiers (exactly, at least, at most, only, every, each) must be interpreted literally.\\
3. For negative criteria (avoid X, must not Y), compliance requires the prohibited element to be entirely absent.\\
4. Judge by semantic content and evidence, not surface keyword matching.\\
5. Provide a concise reason consistent with the compliance decision.

\smallskip
\textbf{Output:} \texttt{\{"satisfied": 1, "reason": "short evidence-based justification"\}}

\tcbline\scriptsize

\textbf{Template 5.}

\smallskip
\textbf{System:} You are a neutral third-party observer asked to verify whether a single evaluation criterion holds for a given response. You must not consider any aspect of quality beyond the stated criterion. Output a JSON object with two fields: satisfied (integer 0 or 1) and reason (string).

\smallskip
\textbf{User:} Below is the user's original request: \texttt{\{prompt\}}; Below is the response that was generated: \texttt{\{response\}}; The evaluation criterion is: \texttt{\{criterion\}}.

\smallskip
\textbf{Your task:} Determine whether the response meets this criterion. Consider only this
criterion and nothing else about the response.

\smallskip
\textbf{Key principles:}\\
- Quantifiers are literal: ``exactly 3'' means precisely 3, not 2 or 4.\\
- Negative requirements (avoid, do not, must not) pass only when the forbidden element is completely absent.\\
- Judge substance and meaning rather than superficial keyword presence.\\
- State a brief reason that directly supports your binary verdict.

\smallskip
Respond with exactly: \texttt{\{"satisfied": 1, "reason": "one-sentence justification"\}}
\end{mybox}

\paragraph{Baselines and Hyperparameters.}
All methods use the same verifier outputs whenever their protocols permit. URW directly uses the calibration-score set $\mathcal{S}$ in Eq.~(\ref{eq:pval_hyp}) and introduces no tunable hyperparameter. For Sampling Ensemble~\citep{wang2023selfconsistency}, we independently decode each rubric six times at temperature $\mathrm{T}=0.7$ and use its vote ratio as the uncertainty signal. For Beta-Binomial, we use a uniform $\mathrm{Beta}(1,1)$ prior over each criterion's satisfaction probability and take the posterior mean as its weight. SCP and Mondrian-SCP use the same template-averaged score $\hat v_k$ as URW, with $\alpha=0.1$; Mondrian-SCP partitions $\hat v_k$ into $\mathrm{B}=6$ bins for bin-wise calibration.

\paragraph{Downstream RLVR Protocol.}
We optimize Qwen 3 8B and GLM 4 9B policies with GRPO~\citep{deepseek-grpo} using the \texttt{VeRL-Rubric} framework\footnote{\url{https://github.com/DoYangTan/verl-rubric}}. Qwen 3 32B serves as the rubric verifier for Qwen 3 8B, and GLM 4 32B serves as the rubric verifier for GLM 4 9B. For a sampled policy response, the verifier evaluates every criterion under the fixed template pool, and the resulting binary judgments are aggregated with the method-specific rule to form the GRPO reward. Uniform, Mondrian-SCP, and URW therefore share the same policy backbone, verifier family, training prompts, and optimization framework; only their rubric-reward aggregation differs. Each policy is optimized for $150$ GRPO steps over the $2{,}000$-prompt training corpus, with an identical step budget across methods. We perform no in-domain fine-tuning on Arena-Hard or GPQA-Diamond. During evaluation, both policies use temperature $1.0$, with maximum output lengths of $16{,}384$ tokens.

\paragraph{GPQA-Diamond Evaluation.}
For each question, the policy receives the following zero-shot chain-of-thought prompt, where the four options are the deterministically shuffled choices described above. The placeholders \texttt{\{question\}}, \texttt{\{A\}}, \texttt{\{B\}}, \texttt{\{C\}}, and \texttt{\{D\}} are replaced with the question and its shuffled answer options, respectively.

\begin{mybox}{Prompt}{GPQA-Diamond zero-shot chain-of-thought prompt used for downstream RLVR evaluation.}
\scriptsize
\textbf{User:} Answer the following multiple choice question. The last line of your response should be of the following format: \texttt{Answer: \$LETTER} (without quotes) where LETTER is one of ABCD. Think step by step before answering.

\smallskip
\texttt{\{question\}}

\smallskip
\texttt{A) \{A\}}\\
\texttt{B) \{B\}}\\
\texttt{C) \{C\}}\\
\texttt{D) \{D\}}
\end{mybox}
We extract the predicted option using ordered format rules, including \texttt{Answer: X}, \texttt{final answer: X}, \texttt{\textbackslash boxed\{X\}}, and a terminal answer letter. We compute accuracy over responses from which an answer letter is successfully extracted; failed extractions are recorded as evaluation errors.

\paragraph{Arena-Hard Evaluation.}
We use DeepSeek V4.1 Flash~\citep{deepseek-v4.1} as the judge with temperature $0.0$ and a maximum of $32{,}768$ tokens. For every prompt, we apply the official Arena-Hard-Auto pairwise-comparison prompt below twice, once with the policy response as Assistant A and once with it as Assistant B, to mitigate position bias. The placeholders \texttt{\{question\}}, \texttt{\{answer\_a\}}, and \texttt{\{answer\_b\}} are replaced with the benchmark prompt and the two responses under comparison, respectively.

\begin{mybox}{Prompt}{Arena-Hard-Auto pairwise judge prompt used for downstream RLVR evaluation.}
\scriptsize
\textbf{User:} Please act as an impartial judge and evaluate the quality of the responses provided by two AI assistants to the user prompt displayed below. You will be given assistant A's answer and assistant B's answer. Your job is to evaluate which assistant's answer is better.

\smallskip
Begin your evaluation by generating your own answer to the prompt. You must provide your answers before judging any answers.

\smallskip
When evaluating the assistants' answers, compare both assistants' answers with your answer. You must identify and correct any mistakes or inaccurate information. Then consider if the assistants' answers are helpful, relevant, and concise. Consider creativity and novelty when needed, and identify important missing information.

\smallskip
After providing your explanation, output one final verdict label: \texttt{[[A>>B]]}, \texttt{[[A>B]]}, \texttt{[[A=B]]}, \texttt{[[B>A]]}, or \texttt{[[B>>A]]}.

\smallskip
\texttt{<|User Prompt|>}\\
\texttt{\{question\}}

\smallskip
\texttt{<|The Start of Assistant A's Answer|>}\\
\texttt{\{answer\_a\}}\\
\texttt{<|The End of Assistant A's Answer|>}

\smallskip
\texttt{<|The Start of Assistant B's Answer|>}\\
\texttt{\{answer\_b\}}\\
\texttt{<|The End of Assistant B's Answer|>}
\end{mybox}
We map \texttt{[[A>>B]]}, \texttt{[[A>B]]}, \texttt{[[A=B]]}, \texttt{[[B>A]]}, and \texttt{[[B>>A]]} to scores of $1$, $1$, $0.5$, $0$, and $0$ for Assistant A, respectively, invert the second game's score because the policy occupies Assistant B, and average all parsable game scores. We report the mean resulting score as win rate; examples with no parsable verdict from either game are recorded as evaluation errors.

\section{More Experimental Results}\label{app:extended}

In this section, we provide two additional experimental results: (1) an extended template-count sensitivity analysis on RM-Bench, RewardBench, and RewardBench~2, and (2) qualitative case studies of downstream RLVR policies on GPQA-Diamond.

\subsection{More Sensitivity Results}\label{app:sensitivity}

In this subsection, we complement the template-count sensitivity analysis in Section~\ref{subsec:hyperparameter} with RM-Bench, RewardBench, and RewardBench~2. Figure~\ref{fig:hparam_N_app} reports accuracy as the template count $\mathrm{N}$ varies from $1$ to $6$ for the same three verifier profiles; as in the main text, $\mathrm{N}=1$ corresponds to Uniform aggregation. We make two observations.
\ding{182} \textbf{Multi-template aggregation consistently improves over Uniform aggregation.} Across the additional benchmarks and verifier backbones, using more than one template generally yields higher accuracy than the single-template Uniform baseline. This pattern extends the main-text results and shows that the benefit of aggregating cross-template verifier behavior is not restricted to a particular benchmark or model family.
\ding{183} \textbf{URW is robust to the exact template count.} Most performance gains emerge with a small template pool, after which the curves largely plateau or vary only modestly. Although individual trajectories can fluctuate as templates contribute overlapping or complementary judgment signals, the overall pattern indicates that URW does not require delicate tuning of $\mathrm{N}$ to obtain its benefit.

\begin{figure}[t]
\centering
\subfigure[RM-Bench]{\includegraphics[width=0.31\linewidth]{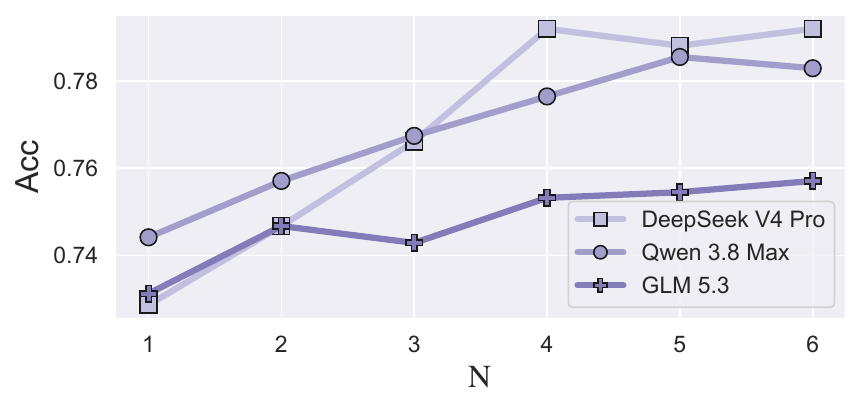}
}
\hfill
\subfigure[RewardBench]{
\includegraphics[width=0.31\linewidth]{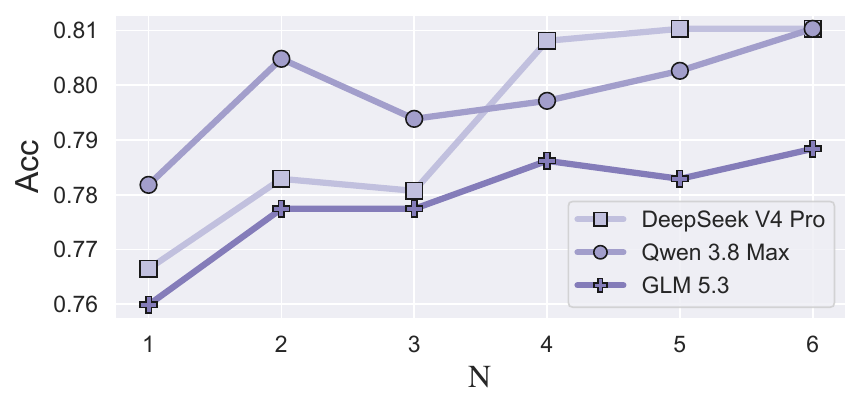}
}
\hfill
\subfigure[RewardBench 2]{
\includegraphics[width=0.31\linewidth]{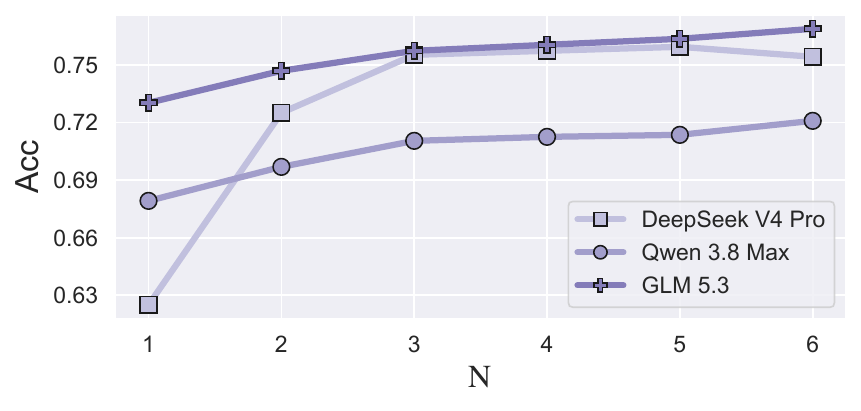}
}
\caption{Sensitivity to the template count $\mathrm{N}$.}\label{fig:hparam_N_app}
\end{figure}

\subsection{Case Study on Downstream RLVR}
\label{subsec:case_study}

In this subsection, we present qualitative case studies illustrating how rubric-reward aggregation translates into downstream policy behavior. We conduct the analysis on GPQA-Diamond.
Tables~\ref{tab:qualitative-case-qwen} and~\ref{tab:qualitative-case-glm} present responses from Qwen 3 8B and GLM 4 9B policies trained with different rubric-reward aggregation methods, respectively. The cases compare Uniform aggregation and uncertainty-aware baselines against URW on questions for which the methods produce different final answers. From the responses, we draw two observations.
\ding{182} \textbf{Policies trained with Uniform aggregation and uncertainty-aware baselines make incorrect final predictions in both cases.} All non-URW policies shown in the tables select an incorrect answer on their respective case. These examples suggest that uniformly or heuristically aggregating criterion-level judgments can yield reward signals that do not consistently favor the reasoning patterns required for the correct answer.
\ding{183} \textbf{Only URW-trained policies produce the correct answer in both cases.} In contrast, the URW-trained Qwen~3~8B and GLM~4~9B policies each identify the correct option after providing task-relevant reasoning. While these qualitative examples do not establish a general performance guarantee, they are consistent with the aggregate GPQA-Diamond results in Section~\ref{subsec:downstream_rlhf} and illustrate how calibrated criterion-level weighting can support downstream scientific reasoning.

{
\scriptsize
\renewcommand{\arraystretch}{1.1}
\setlength{\tabcolsep}{3pt}
\begin{longtable}{p{1.9cm}p{14.6cm}}
\caption{Qualitative case study on GPQA-Diamond with Qwen 3 8B.}\label{tab:qualitative-case-qwen}\\
\toprule
\textbf{Prompt}: &
Answer the following multiple choice question. The last line of your response should be of the following format: 'Answer: \$LETTER' (without quotes) where LETTER is one of ABCD. Think step by step before answering.

\smallskip
Sigmatropic rearrangement is a class of pericyclic reactions. In these reactions, the terminal pi bond migrates into the sigma bond. Cope and Claisen rearrangements are also sigma tropic rearrangements. These reactions are usually thermodynamically favored. Complete the following reactions.

\smallskip
1,1-dimethoxyethan-1-amine + but-3-en-2-ol + (H+ + Heat) $\rightarrow$ A

(3R,4S)-3,4-dimethylhexa-1,5-diyne + Heat $\rightarrow$ B

2-((vinyloxy)methyl)but-1-ene + Heat $\rightarrow$ C

\smallskip
A) A = 6-methyl-3,4-dihydro-2H-pyran-2-amine, B = (1Z,2E)-1,2-diethylidenecyclobutane, C = 4-methylenehexanal

B) A = 6-methyl-3,4-dihydro-2H-pyran-2-amine, B = (1Z,2E)-1,2-diethylidenecyclobutane, C = 4-methylenehexan-1-ol

C) A = (Z)-1-(but-2-en-2-yloxy)ethen-1-amine, B = (3Z,4E)-3,4-diethylidenecyclobut-1-ene, C = 4-methylenehexan-1-ol

D) A = (Z)-1-(but-2-en-2-yloxy)ethen-1-amine, B = (3Z,4E)-3,4-diethylidenecyclobut-1-ene, C = 4-methylenehexanal
\\
\hdashline
\textbf{True Answer}: & D \\
\midrule

\multicolumn{2}{l}{\textbf{Responses}:} \\
\endfirsthead

\multicolumn{2}{r}{\scriptsize\emph{(Table~\ref{tab:qualitative-case-qwen} continued from previous page)}}\\
\toprule
\multicolumn{2}{l}{\textbf{Responses} (continued):} \\
\midrule
\endhead

\midrule
\multicolumn{2}{r}{\scriptsize\emph{Continued on next page}}\\
\endfoot

\bottomrule
% \multicolumn{2}{p{13.5cm}}{\leftskip=1em\rightskip=-1em\scriptsize\textit{Note}: \prompt{Prompt is wrapped in blue}. \harmful{Harmful segments are highlighted in red}, while \benign{CausalRM refusal is highlighted in green}. \analysis{Analysis is highlighted in orange}.}\\
\endlastfoot

\textbf{Uniform}: &

To determine the correct answer, we analyze each reaction step by step, focusing on the mechanisms and possible products of the sigmatropic rearrangements involved.

---

\#\#\# Reaction A:

Starting Materials:  

- 1,1-dimethoxyethan-1-amine  

- but-3-en-2-ol (2-hydroxybut-3-ene)  

- Conditions: H$^+$ and heat  

Analysis:  

This reaction involves the formation of a cyclic ether, specifically a pyran ring (a six-membered ring containing one oxygen atom). The presence of two methoxy groups on the nitrogen-substituted carbon and the hydroxyl group from the alcohol facilitate the formation of a cyclic hemiaminal or pyran via an acid-catalyzed rearrangement. The product is 6-methyl-3,4-dihydro-2H-pyran-2-amine, a pyran ring with a methyl group on C6 and an amine group on C2. This matches a Claisen-type rearrangement where the hydroxyl group and the amine group participate in ring formation.

---

\#\#\# Reaction B:

Starting Material:  

- (3R,4S)-3,4-dimethylhexa-1,5-diyne (a six-carbon chain with two triple bonds and methyl groups at positions 3 and 4)  

- Conditions: Heat  

Analysis:  

This compound undergoes a [3,3]-sigmatropic rearrangement, forming a cyclobutane ring. The two triple bonds (C1–C2 and C5–C6) become double bonds within the ring, and the methyl groups at positions 3 and 4 migrate to form the ring. The product (1Z,2E)-1,2-diethylidenecyclobutane is formed, where the ring consists of two adjacent double bonds, each with Z and E configurations, matching the stereochemistry of the starting compound. The Z and E configuration in the product aligns with the stereochemical preservation during the reaction.

---

\#\#\# Reaction C:

Starting Material:  

- 2-((vinyloxy)methyl)but-1-ene  

- Conditions: Heat  

Analysis:  

This is a vinyloxy group (an ether group with a double bond) attached to a methyl group on carbon 2 of but-1-ene. Under heat, a [1,5]-sigmatropic rearrangement occurs, where the vinyloxy group shifts to form a new double bond. The oxygen from the vinyloxy group can act as a nucleophile, leading to the formation of an aldehyde group (as the oxygen becomes part of a carbonyl group). The product is 4-methylenehexanal, which has a methylene group at position 4 and an aldehyde group at position 1. This aligns with the Claisen-type rearrangement mechanism where the oxygen becomes part of a carbonyl.

---

\#\#\# Conclusion:

- Reaction A forms a pyran ring, matching Option A's product.  

- Reaction B forms a cyclobutane ring with Z\/E configurations, matching Option A's product.  

- Reaction C forms an aldehyde, matching Option A's product.  

All three reactions align with the product described in Option A.

---

Answer: A
\\

\hdashline
\textbf{Predicted Answer}: & A \\
\midrule

\textbf{Beta-Binomial}: &

To determine the correct answer, we analyze each of the three reactions using the principles of sigmatropic rearrangements:

---

\#\#\# Reaction A: 1,1-dimethoxyethan-1-amine + but-3-en-2-ol + (H+ + Heat) → A

This is a classic example of a Claisen-type rearrangement. The but-3-en-2-ol (an allylic alcohol) and the 1,1-dimethoxyethan-1-amine (likely a nucleophilic amine with two methoxy groups) undergo an acid-catalyzed, thermal rearrangement. The methoxy groups and hydroxyl group from the alcohol combine to form a six-membered cyclic ether (pyran), with an amine group at position 2 of the ring. This matches Option B's A: 6-methyl-3,4-dihydro-2H-pyran-2-amine, which features the correct ring structure and functional group placements.

---

\#\#\# Reaction B: (3R,4S)-3,4-dimethylhexa-1,5-diyne + Heat → B

The starting compound is a dihydrocarbon with two methyl groups on carbons 3 and 4 and triple bonds at positions 1 and 5. Heating under such conditions typically leads to a [1,5]-sigmatropic rearrangement or ring-forming reaction. Given the stereochemistry (R and S configuration on the methyl groups), and the need to form a cyclobutane ring, the product is (1Z,2E)-1,2-diethylidenecyclobutane, which aligns with the stereochemistry and geometry of the starting material. This matches Option B's B.

---

\#\#\# Reaction C: 2-((vinyloxy)methyl)but-1-ene + Heat → C

The starting material has a vinyloxy group attached to a methyl group on carbon 2 of but-1-ene. Upon heating, a [1,2] sigmatropic rearrangement is likely to occur, where the vinyloxy group forms a carbonyl group and extends the chain. The final product would be a hexanal with a methylene group at position 4, matching Option B's C: 4-methylenehexan-1-ol (alcohol) or Option D's C: 4-methylenehexanal (aldehyde). Given the oxygen from the vinyloxy group and the possibility of forming a carbonyl (aldehyde), 4-methylenehexanal is a more plausible product than an alcohol. This suggests Option D for this reaction.

---

\#\#\# Final Assessment:

- A: Matches Option B

- B: Matches Option B

- C: Better aligned with Option D

Thus, the correct option is:

Answer: B \\

\hdashline
\textbf{Predicted Answer}: & B \\
\midrule

\textbf{Sampling Ensemble}: &  Answer: B \\
\hdashline
\textbf{Predicted Answer}: & B \\
\midrule

\textbf{SCP}: &
To solve this problem, we analyze each reaction in the context of sigmatropic rearrangements, a type of pericyclic reaction characterized by the migration of a sigma bond, often involving the formation or breaking of a cyclic structure.

---

\#\#\# Reaction A:

Reactants:  

- 1,1-dimethoxyethan-1-amine  

- But-3-en-2-ol  

- H$^+$ + Heat  

This is a classic Claisen-type sigmatropic rearrangement, where the allylic alcohol (but-3-en-2-ol) reacts with an amine (with dimethoxy groups) to form a cyclic pyran ring (six-membered ring with one oxygen). The dimethoxy groups likely act as an ether-forming component, and the amine becomes part of the new ring. The final product forms a 3,4-dihydro-2H-pyran-2-amine with a 6-methyl substituent, forming the structure 6-methyl-3,4-dihydro-2H-pyran-2-amine. This matches Option A and B for this part.

---

\#\#\# Reaction B:

Reactants:  

- (3R,4S)-3,4-dimethylhexa-1,5-diyne  

- Heat  

This molecule undergoes a [2,3] sigmatropic rearrangement or cyclization, where the two triple bonds at positions 1 and 5 and the methyl groups at positions 3 and 4 form a cyclobutane. The stereochemistry (R\/S) of the original molecule is preserved, leading to a product such as (1Z,2E)-1,2-diethylidenecyclobutane, as the double bonds in the ring are arranged to maintain the Z\/E configuration. This aligns with Option A and B.

---

\#\#\# Reaction C:

Reactant:  

- 2-((vinyloxy)methyl)but-1-ene  

- Heat  

This is likely a [1,5]-sigmatropic rearrangement (a type of Cope or Claisen rearrangement), where the vinyloxy group migrates to form a methylene group adjacent to an aldehyde group. The result is a product like 4-methylenehexanal (with a methylene group at position 4 and the aldehyde group at the end of the chain), matching Option A and D.

---

\#\#\# Conclusion:

- Reaction A produces a 6-methyl-3,4-dihydro-2H-pyran-2-amine.

- Reaction B produces (1Z,2E)-1,2-diethylidenecyclobutane.

- Reaction C yields 4-methylenehexanal.

All three match Option A.

---

Answer: A
\\
\hdashline
\textbf{Predicted Answer}: & A \\
\midrule

\textbf{Mondrian-SCP}: & Answer: A \\
\hdashline
\textbf{Predicted Answer}: & A \\
\midrule

\textbf{URW}: &
The reactions described involve sigmatropic rearrangements, which are pericyclic reactions. Let's analyze each reaction and the given options:

1. Reaction A involves 1,1-dimethoxyethan-1-amine and but-3-en-2-ol under acidic and thermal conditions. This setup suggests a Claisen rearrangement, where an allylic alcohol (but-3-en-2-ol) and a nucleophile (amine) form a cyclic ether. The product should be a pyran ring (six-membered ring with an oxygen) with an amine group. The correct product is 6-methyl-3,4-dihydro-2H-pyran-2-amine, as seen in Option A.

2. Reaction B involves (3R,4S)-3,4-dimethylhexa-1,5-diyne under heat. The presence of two triple bonds and methyl groups on carbons 3 and 4 hints at a [1,5] sigmatropic rearrangement, likely forming a cyclobutane. The correct product is (1Z,2E)-1,2-diethylidenecyclobutane, matching Option A and Option B.

3. Reaction C involves 2-((vinyloxy)methyl)but-1-ene under heat. This is a Claisen-like rearrangement of an allyl vinyl ether, which typically forms an aldehyde. The product is 4-methylenehexanal, which aligns with Option D (C = 4-methylenehexanal).

Combining these analyses:

- Reaction A matches Option A.

- Reaction B matches Option A and Option B.

- Reaction C matches Option D.

However, Option A and Option B only partially match, while Option D correctly addresses Reaction C. Since Option A provides the correct product for Reaction A and Reaction B (as per the cyclobutane structure), and Option D is incorrect for Reaction C, the correct answer must reconcile these. However, Option D is the only one where Reaction C is correctly identified as forming an aldehyde, which is crucial for a complete answer.

Answer: D
\\
\hdashline
\textbf{Predicted Answer}: & D \\
\end{longtable}
}

{
\scriptsize
\renewcommand{\arraystretch}{1.1}
\setlength{\tabcolsep}{3pt}
\begin{longtable}{p{1.9cm}p{14.6cm}}
\caption{Qualitative case study on GPQA-Diamond with GLM 4 9B.}\label{tab:qualitative-case-glm}\\
\toprule
\textbf{Prompt}: &
Answer the following multiple choice question. The last line of your response should be of the following format: 'Answer: \$LETTER' (without quotes) where LETTER is one of ABCD. Think step by step before answering.

\smallskip
A textile dye containing an extensively conjugated pi-electrons emits light with energy of 2.3393 eV. What color of light is absorbed by the organic compound?

\smallskip
A) Red

B) Blue

C) Violet

D) Yellow
\\
\hdashline
\textbf{True Answer}: & A \\
\midrule

\multicolumn{2}{l}{\textbf{Responses}:} \\
\endfirsthead

\multicolumn{2}{r}{\scriptsize\emph{(Table~\ref{tab:qualitative-case-glm} continued from previous page)}}\\
\toprule
\multicolumn{2}{l}{\textbf{Responses} (continued):} \\
\midrule
\endhead

\midrule
\multicolumn{2}{r}{\scriptsize\emph{Continued on next page}}\\
\endfoot

\bottomrule
% \multicolumn{2}{p{13.5cm}}{\leftskip=1em\rightskip=-1em\scriptsize\textit{Note}: \prompt{Prompt is wrapped in blue}. \harmful{Harmful segments are highlighted in red}, while \benign{CausalRM refusal is highlighted in green}. \analysis{Analysis is highlighted in orange}.}\\
\endlastfoot

\textbf{Uniform}: &
To determine the color of light absorbed by the organic compound, we need to follow these steps:

1. Convert the energy from eV to wavelength (nm):

   The energy of a photon is related to its wavelength by the equation: $ E = \frac{hc}{\lambda} $ where:

   - $ E $ is the energy of the photon (in eV),

   - $ h $ is Planck's constant ($ 4.135667696 \times 10^{-15} $ eV·s),

   - $ c $ is the speed of light ($ 2.99792458 \times 10^8 $ m\/s),

   - $ \lambda $ is the wavelength (in meters).

   Rearranging for $ \lambda $: $ \lambda = \frac{hc}{E} $ Plugging in the values: $ \lambda = \frac{(4.135667696 \times 10^{-15} \, \text{eV·s}) \times (2.99792458 \times 10^8 \, \text{m\/s})}{2.3393 \, \text{eV}}
$ $ \lambda = \frac{1.241440748 \times 10^{-6} \, \text{m·eV}}{2.3393 \, \text{eV}}
$ $ \lambda = 5.2975 \times 10^{-7} \, \text{m} = 529.75 \, \text{nm}
$

2. Interpret the wavelength:

   - Visible light ranges from approximately 400 nm (violet) to 700 nm (red).

   - A wavelength of 529.75 nm falls in the green-blue region, but closer to blue.

   However, the options provided are:

   - A) Red (700 nm)

   - B) Blue (450-495 nm)

   - C) Violet (380-450 nm)

   - D) Yellow (570-590 nm)

   Since 529.75 nm is closer to blue than yellow, the closest match is B) Blue.

Answer: B
\\

\hdashline
\textbf{Predicted Answer}: & B \\
\midrule

\textbf{Beta-Binomial}: &
To determine the color of light absorbed by the organic compound, we need to follow these steps:

1. Convert energy from eV to wavelength: The energy of light is related to its wavelength by the equation: $ E = \frac{hc}{\lambda} $ where:

   - $ E $ is the energy in joules (J),

   - $ h $ is Planck's constant ($ 6.626 \times 10^{-34} \, \text{J} \cdot \text{s} $),

   - $ c $ is the speed of light ($ 3.00 \times 10^8 \, \text{m\/s} $),

   - $ \lambda $ is the wavelength in meters (m).

   Since the energy is given in electron volts (eV), we first convert it to joules: $ 1 \, \text{eV} = 1.602 \times 10^{-19} \, \text{J} $ So: $ E = 2.3393 \, \text{eV} \times 1.602 \times 10^{-19} \, \text{J\/eV} = 3.746 \times 10^{-19} \, \text{J}
$

2. Solve for the wavelength $ \lambda $: $ \lambda = \frac{hc}{E} $ Substituting the values: $ \lambda = \frac{(6.626 \times 10^{-34} \, \text{J} \cdot \text{s}) \times (3.00 \times 10^8 \, \text{m\/s})}{3.746 \times 10^{-19} \, \text{J}}
$ $ \lambda = \frac{1.9878 \times 10^{-25} \, \text{J} \cdot \text{m}}{3.746 \times 10^{-19} \, \text{J}}
$ $ \lambda \approx 5.3 \times 10^{-7} \, \text{m} = 530 \, \text{nm}
$

3. Determine the color of light absorbed: The visible light spectrum ranges from approximately 400 nm (violet) to 700 nm (red). A wavelength of 530 nm falls in the green region, but since yellow is the closest option, we choose yellow. However, the exact answer is not listed, so we need to choose the closest match.

   The closest color to 530 nm is yellow, so the compound absorbs yellow light.

Answer: D
\\

\hdashline
\textbf{Predicted Answer}: & D \\
\midrule

\textbf{Sampling Ensemble}: & 
To determine the color of light absorbed by the organic compound, we need to follow these steps:

1. Understand the relationship between energy and wavelength: The energy of a photon is inversely proportional to its wavelength. The formula is: $ E = \frac{hc}{\lambda} $ where $ E $ is energy, $ h $ is Planck's constant, $ c $ is the speed of light, and $ \lambda $ is the wavelength.

2. Convert energy from eV to Joules: The energy given is 2.3393 eV. We know that: $ 1 \text{ eV} = 1.602 \times 10^{-19} \text{ Joules} $ So, the energy in Joules is: $ E = 2.3393 \times 1.602 \times 10^{-19} = 3.753 \times 10^{-19} \text{ Joules}
$

3. Calculate the wavelength: Using the formula $ E = \frac{hc}{\lambda} $, we can rearrange it to find $ \lambda $: $ \lambda = \frac{hc}{E} $ Plugging in the values: $ h = 6.626 \times 10^{-34} \text{ J·s}, \quad c = 3.00 \times 10^8 \text{ m\/s}
$ $ \lambda = \frac{(6.626 \times 10^{-34})(3.00 \times 10^8)}{3.753 \times 10^{-19}} = 5.32 \times 10^{-7} \text{ meters} = 532 \text{ nm}
$

4. Determine the color of light: The visible light spectrum is roughly:

   - Red: 620-750 nm

   - Orange: 590-620 nm

   - Yellow: 570-590 nm

   - Green: 495-570 nm

   - Blue: 450-495 nm

   - Violet: 380-450 nm

   A wavelength of 532 nm falls in the green region, but since the question asks for the color absorbed (which corresponds to the color we see, as the compound absorbs that color and reflects others), and 532 nm is closest to yellow-green, the most likely answer is yellow (as yellow is the complementary color to blue\/green).

However, based on the options, the closest match is yellow.

Answer: D
\\
\hdashline
\textbf{Predicted Answer}: & D \\
\midrule

\textbf{SCP}: &
To determine the color of light absorbed by the organic compound, we first need to convert the given energy from electron volts (eV) to wavelength (in nanometers), and then match the wavelength to the corresponding color in the visible light spectrum.

Step 1: Convert the energy from eV to joules (J). $1 \, \text{eV} = 1.602 \times 10^{-19} \, \text{J}$ $E = 2.3393 \, \text{eV} \times 1.602 \times 10^{-19} \, \text{J\/eV}$ $E = 3.755 \times 10^{-19} \, \text{J}$

Step 2: Use the equation $ E = \frac{hc}{\lambda} $ to find the wavelength ($\lambda$). $h = 6.626 \times 10^{-34} \, \text{J·s}$ $c = 3.00 \times 10^8 \, \text{m\/s}$ $\lambda = \frac{hc}{E}$ $\lambda = \frac{(6.626 \times 10^{-34} \, \text{J·s}) \times (3.00 \times 10^8 \, \text{m\/s})}{3.755 \times 10^{-19} \, \text{J}}$ $\lambda = \frac{1.9878 \times 10^{-25} \, \text{J·m}}{3.755 \times 10^{-19} \, \text{J}}$ $\lambda = 5.285 \times 10^{-7} \, \text{m}$ $\lambda = 528.5 \, \text{nm}$

Step 3: Match the wavelength to the color of light.

- Red: 620-750 nm

- Orange: 590-620 nm

- Yellow: 570-590 nm

- Green: 495-570 nm

- Blue: 450-495 nm

- Violet: 380-450 nm

A wavelength of 528.5 nm falls within the green region of the visible light spectrum, but since the closest option is yellow (570-590 nm), we need to check the options again. However, the given options do not include green, and the closest is yellow. This suggests a possible discrepancy in the options or the interpretation, but based on the given choices, yellow is the closest match.

Answer: D
\\
\hdashline
\textbf{Predicted Answer}: & D \\
\midrule

\textbf{Mondrian-SCP}: & 
To determine the color of light absorbed by the organic compound, we need to understand the relationship between the energy of absorbed light and the color. Here's the step-by-step reasoning:

1. Energy and Wavelength Relationship: The energy of light is inversely proportional to its wavelength. This relationship is given by the equation: $ E = \frac{hc}{\lambda} $ where:

   - $ E $ is the energy of light (in eV),

   - $ h $ is Planck's constant ($ 4.135667696 \times 10^{-15} $ eV·s),

   - $ c $ is the speed of light ($ 2.99792458 \times 10^8 $ m\/s),

   - $ \lambda $ is the wavelength of light (in meters).

2. Calculate Wavelength: Using the given energy ($ E = 2.3393 $ eV), we can calculate the wavelength: $ \lambda = \frac{hc}{E} = \frac{(4.135667696 \times 10^{-15} \, \text{eV·s}) \times (2.99792458 \times 10^8 \, \text{m\/s})}{2.3393 \, \text{eV}}
$ $ \lambda \approx 5.3004 \times 10^{-7} \, \text{m} = 530.04 \, \text{nm}
$

3. Interpret Wavelength and Color: The visible spectrum ranges from approximately 400 nm (violet) to 700 nm (red). A wavelength of 530.04 nm falls within the greenish-yellow region, but since this is not an option, we need to consider the closest match. The closest option to 530.04 nm is yellow.

4. Check Options: The options provided are:

   - A) Red (approx. 700 nm)

   - B) Blue (approx. 475 nm)

   - C) Violet (approx. 400 nm)

   - D) Yellow (approx. 570 nm)

   The wavelength of 530.04 nm is closest to yellow.

Answer: D
\\
\hdashline
\textbf{Predicted Answer}: & D \\
\midrule

\textbf{URW}: &
To determine the color of light absorbed by the organic compound, we need to follow these steps:

1. Calculate the wavelength of the emitted light:

   The energy of the emitted light is given as 2.3393 eV. We can convert this energy to wavelength using the formula: $ E = \frac{hc}{\lambda} $ where:

   - $ E $ is the energy (in joules),

   - $ h $ is Planck's constant ($ 6.626 \times 10^{-34} \, \text{J} \cdot \text{s} $),

   - $ c $ is the speed of light ($ 3.00 \times 10^8 \, \text{m\/s} $),

   - $ \lambda $ is the wavelength (in meters).

   First, convert the energy from eV to joules: $ 1 \, \text{eV} = 1.602 \times 10^{-19} \, \text{J}
$ $ E = 2.3393 \, \text{eV} \times 1.602 \times 10^{-19} \, \text{J\/eV} = 3.754 \times 10^{-19} \, \text{J} $ Now, solve for $ \lambda $: $ \lambda = \frac{hc}{E} = \frac{(6.626 \times 10^{-34} \, \text{J} \cdot \text{s}) \times (3.00 \times 10^8 \, \text{m\/s})}{3.754 \times 10^{-19} \, \text{J}}
$ $ \lambda = \frac{1.9878 \times 10^{-25} \, \text{J} \cdot \text{m}}{3.754 \times 10^{-19} \, \text{J}} = 5.29 \times 10^{-7} \, \text{m} = 529 \, \text{nm}
$

2. Determine the color of light absorbed:

   The color of light absorbed is complementary to the color of light emitted. The visible spectrum is divided as follows:

   - Red: 620–750 nm

   - Orange: 590–620 nm

   - Yellow: 570–590 nm

   - Green: 495–570 nm

   - Blue: 450–495 nm

   - Violet: 380–450 nm

   A wavelength of 529 nm falls within the green region of the visible spectrum. Therefore, the organic compound absorbs red light (which is complementary to green light).

3. Select the correct answer:

   The color of light absorbed by the organic compound is red.

Answer: A
\\
\hdashline
\textbf{Predicted Answer}: & A \\
\end{longtable}
}

\end{document}

%% file: math_commands.tex
\usepackage{amsmath,amsfonts,bm}

\def\eqref#1{equation~\ref{#1}}
\def\1{\bm{1}}

\DeclareMathAlphabet{\mathsfit}{\encodingdefault}{\sfdefault}{m}{sl}
\SetMathAlphabet{\mathsfit}{bold}{\encodingdefault}{\sfdefault}{bx}{n}